\documentclass[lettersize,journal]{IEEEtran}
\usepackage{amsmath,amsfonts}
\usepackage{array}
\usepackage[caption=false,font=footnotesize]{subfig}
\usepackage{textcomp}
\usepackage{stfloats}
\usepackage{url}
\usepackage{verbatim}
\usepackage{graphicx}
\usepackage{balance}
\usepackage{amssymb}
\usepackage{bm}
\usepackage{tikz}
\usepackage{multirow}
\usepackage{ulem}  % \sout{}
\newtheorem{theorem}{Theorem}

\usepackage{xcolor}
\definecolor{lightbrown}{RGB}{166,103,56}

\begin{document}
\title{A Three-Stage Offline Control for High Fidelity Motion Reproduction of Human Lower Limb on a Suspended Bipedal Robot}

\author{Ping-Kong Huang, %~\IEEEmembership{Student Member,~IEEE,}
Chien-Wu Lan, %~\IEEEmembership{Member,~IEEE,}
Ching-Kai Lin, 
and Chin-Tien Wu
\thanks{This work was supported by the National Science and Technology Council (NSTC), Taiwan, under Grant No.~114-2115-M-A49-017-MY2. (Corresponding author: Chin-Tien Wu.)}
\thanks{Ping-Kong Huang, Ching-Kai Lin, and Chin-Tien Wu are with the Department of Applied Mathematics, National Yang Ming Chiao Tung University, Hsinchu 30010, Taiwan (e-mail: s0752307.sc07@nycu.edu.tw; magicjacky0130.sc11@nycu.edu.tw; ctw@math.nctu.edu.tw).}
\thanks{Chien-Wu Lan is with the Department of Electrical Engineering, National Central University, Taoyuan City 320317, Taiwan (e-mail: chienwulan@g.ncu.edu.tw).}
}

\maketitle

\begin{abstract}
This paper presents a three-stage offline command generation framework for reproducing human lower-limb motion on a suspended bipedal robot while matching torque trajectories computed from the robot dynamic model. First, State-Dependent Riccati Equation (SDRE) control derives the reference torque trajectory for the measured motion. Second, parameterized optimization converts this trajectory into trapezoidal joint velocity commands under motor speed and acceleration limits. Third, a proportional-integral-derivative linear quadratic regulator (PID-LQR) compensation scheme refines these commands using experimental tracking data. The platform executes the resulting profiles to reproduce human walking and squatting motions recorded by a Vicon system, allowing evaluation of tracking accuracy and repeatability. Results show that the average root mean square error (RMSE) and standard deviation (STD) of joint angles across repeated trials remain below $7^\circ$ and $0.33^\circ$, respectively. Joint angle and torque trajectory comparisons show lower maximum RMSE and STD values than those for MPC and IPSO-PID in every reported case. The framework enables accurate and repeatable motion reproduction within actuator limits, providing controlled and measurable conditions that can reduce reliance on human participation and associated risks during preliminary evaluation of devices for assistive walking, gait training, and rehabilitation.
\end{abstract}


\section{Introduction}

\IEEEPARstart{L}{ower-limb} exoskeletons are designed to assist lower limb motion by using wearer movements as control inputs.
Although these systems are relevant to industrial and rehabilitation applications, early evaluation still raises safety and reliability concerns.
Mechanical malfunction, inconsistent motion replication, joint misalignment, and constrained knee mechanisms can expose the wearer to injury, discomfort, or parasitic loads~\cite{He:2017,Sarkisian:2021,Hong:2023}.
This risk becomes more critical when experiments involve users with physical disabilities or limited tolerance for exploratory tests, since rehabilitation studies with lower limb exoskeleton assistance can remain preliminary and limited by small subject cohorts~\cite{Dai:2025}.
These concerns indicate that exoskeleton evaluation requires controlled and repeatable test procedures before extensive human subject experiments are conducted.

To reduce these risks, robotic test platforms separate the early evaluation process from direct human participation.
Several studies have developed different testing methods for this purpose.
These include dedicated and controllable devices for measuring exoskeleton torque~\cite{Hartmann:2021}, artificial legs or lower-limb simulators that replace human subjects during testing~\cite{Gao:2021,Massardi:2023}, and humanoid robot frameworks that combine simulation and human motion modeling for system evaluation~\cite{Wehrle:2022}.
Besides improving safety during early-stage testing, these platforms also provide measurable and repeatable motion conditions~\cite{Young:2017}, which are difficult to achieve in experiments with human subjects.

Robotic test platforms are generally designed to reproduce the human movements obtained from motion capture systems, making accurate trajectory replication a critical requirement~\cite{Young:2017}.
However, this conversion is not straightforward because captured motion trajectories provide kinematic references, whereas robot execution also depends on dynamic modeling, command generation, and actuator constraints.
Consequently, this conversion must account for platform characteristics, including the mechanical structure, system mass, actuator properties, and motion constraints imposed by the mechanism~\cite{Andrade:2021,Schrade:2021}.
Since these factors influence the physical behavior of the platform, system dynamics must be considered before the recorded trajectories can be converted into executable robot commands~\cite{Buschmann:2007,Abdullah:2022}. 
Various nonlinear and dynamic model-based control frameworks have been used to manage tracking discrepancies on robotic test platforms ~\cite{Islam:2022,Sun:2021,ElHussieny:2024}. Additionally, methods derived from linear quadratic control (LQR) have been applied to stabilize a Pendubot ~\cite{Pazderski:2022} and control locomotion of a biped robot~\cite{Shafei:2025}.\\
% \textcolor{red}{Robotic test platforms are generally designed to reproduce the human movements obtained from  motion capture systems, making accurate trajectory replication a critical requirement ~\cite{Young:2017}.}
% %Therefore, reproducing the recorded motion is an essential requirement in the development of a robotic test platform.
% However, this conversion is not straightforward because captured motion trajectories provide kinematic references, whereas robot execution also depends on dynamic modeling, command generation, and actuator constraints.
% For a suspended bipedal robot platform, these factors include the mechanical structure, system mass, actuator properties, and mechanism-imposed motion constraints~\cite{Andrade:2021,Schrade:2021}.
% Since these factors influence the physical behavior of the platform, system dynamics must be considered before the recorded trajectories can be converted into executable robot commands~\cite{Buschmann:2007,Abdullah:2022}.
 %where bipedal robots are used for lower limb motion reproduction. %are representative examples.
%Dynamic model based control and nonlinear control methods have been investigated for biped robots, lower limb robotic systems, and related tracking problems~\cite{Islam:2022,Sun:2021,ElHussieny:2024}.
It is well-known that the State Dependent Riccati Equation (SDRE) approach represents nonlinear dynamics in a state dependent coefficient form and obtains a nonlinear control feedback law through state dependent Riccati equations~\cite{Alla:2023}.
In this study, the SDRE formulation is utilized to compute sub-optimal control inputs along the reference path, generating the reference torque profiles required to describe the dynamic demands of the captured lower-limb motion. %computes a state dependent control solution along the reference motion and generates the reference torque trajectory required to describe the dynamic demand of the captured lower limb motion.
The resulting torque trajectory therefore specifies the dynamic demand required to reproduce the measured lower-limb motion on the robot platform. However, in general, these theoretical sub-optimal torques cannot be generated by motor actuators precisely due to physical limitations of the actuators, including torque saturation and operating constraints ~\cite{Ghoreishi:2022}. \\
To bridge this gap, parameterized trajectory generation, specifically utilizing trapezoidal velocity or acceleration profiles of motor actuators, offers a practical means to constrain commands within safe actuator velocity and acceleration boundaries. Trapezoidal velocity profiles and optimized acceleration profiles have been shown to be used to generate robot trajectories and motion gait reproduction, respectively, while satisfying kinematic and dynamic constraints~\cite{Chettibi:2006,Cusimano:2022,Khan:2022}. Trajectory tracking under joint state restrictions has been investigated for a suspended bipedal robot~\cite{Rincon:2022}. This prior work supports the use of the suspended configuration for controlled joint motion, while the present study uses this configuration to isolate joint trajectory reproduction before ground interaction and complex exoskeleton coupling are introduced. 
In this study, we take a similar approach introduced in ~\cite{Rincon:2022} to convert the torque reference generated by SDRE into executable velocity and acceleration command sequences while preserving the torque demand required for motion reproduction on a suspended bipedal robot. We address the system-level command generation problem on a suspended bipedal robot platform developed by Lan et al.~\cite{Lan:2021a,Lan:2021b}. To drive this platform with realistic human motion, biological gait data are explicitly captured via the Vicon motion-capture system in this study. Since this optical tracking system only provides joint kinematic trajectories without information from the continuous ground reaction force (GRF), correspondingly, the suspended bipedal robot is also completely free from the influence of GRF. By decoupling these contact dynamics from the control loop, our framework avoids the unpredictable feedback associated with ground impacts.
As a result, the framework focuses on reproducing joint angle trajectories that are compatible with the platform's four actuated degrees of freedom and motor constraints, while establishing repeatable experimental conditions before interaction force modeling is incorporated in future hardware integrations.
Consequently, the objective of this paper is not to directly evaluate the interaction of the human exoskeleton.
Instead, the objective is to develop a control method that enables the suspended bipedal platform to provide this repeatable trajectory baseline for later evaluation of the exoskeleton.

To achieve this objective, this paper proposes a three-stage offline control architecture for feasible motion reproduction by actuators.
In the first stage, the SDRE control method is applied to the dynamics of the robot to generate the joint torque reference required to reproduce the captured gait.
This stage provides a model-based representation of the torque demand associated with the desired lower limb motion.
In the second stage, parameterized optimization converts the torque reference into trapezoidal angular velocity commands that satisfy the constraints of the speed and acceleration of the motor.
This stage maps the model based reference into the actuator command space of the suspended robot.
In the third stage, a PID-LQR acceleration compensation scheme, following the reformulation approaches proposed by O'Brien and Howe~\cite{OBrien:2008} and He, Wang, and Lee~\cite{He:2000}, refines the command profiles using experimental feedback.
This offline formulation is adopted because real time feedback from joint sensors or external measurement systems can be affected by latency, noise, and synchronization error.
If these uncertain measurements are used to update the command during each trial, the resulting motion may vary between repetitions.
This variability is undesirable for a standardized test bench, where the executed motion should remain consistent across trials.

The experimental evaluation compares the proposed framework with baseline controllers,  including the model predictive control (MPC)~\cite{Chen:2018} and PID control tuned using improved particle swarm optimization (IPSO-PID)~\cite{Liu:2021},  under identical motion reproduction conditions.

The main contributions of this study are summarized as follows:
\begin{itemize}
    \item A suspended bipedal robot system is formulated as a surrogate motion reproduction platform that converts Vicon captured lower limb motion into repeatable robot joint motion for walking and squatting tasks.
    \item A three-stage offline command generation architecture is developed by combining SDRE torque reference generation, parameterized motor command optimization, and PID-LQR acceleration compensation under actuator constraints.
    \item 
    In repeated motion reproduction experiments, the proposed three-stage control method reduces the maximum RMSE and STD of the joint angles by at least 62.3\% and 56.1\%, respectively, compared to baseline controllers. Furthermore, the maximum RMSE and STD of model-based joint torques are reduced by at least 61.3\% and 64.5\%, respectively, demonstrating superior repeatability across trials.
\end{itemize}

The remainder of this paper is organized as follows.
The experimental platform and dynamic model are presented in Section II.
The control methodologies, including SDRE control, parameterized optimization, and PID-LQR compensation, are introduced in Section III.
The experimental setup and results are presented in Section IV, including comparisons with the MPC and IPSO-PID baselines.
Finally, conclusions are drawn in Section V.

\section{Dynamic Modeling}
The bipedal robot platform used in this study is shown in Fig.~\ref{fig:model}(a). The robot consists of two legs, each with two actuated joints (hip and knee), providing a total of four degrees of freedom (DoF) in the sagittal plane. Since knee-type exoskeletons typically employ a single actuator per joint, each joint of the bipedal robot is also designed with one degree of freedom to enable one-to-one matching with actuator characteristics. This simple yet functional setup allows the platform to serve as a controlled motion reproduction system for preliminary studies, thus reducing the physical burden and safety risks associated with early-stage system validation. The bipedal robot can be represented by the lower limb model shown in Fig.~\ref{fig:model}(b). To further reduce the complexity of the model, several assumptions are applied to our bipedal robot. 
First, joint movement is constrained to the sagittal plane. 
Second, the centers of mass for the thigh and calf are assumed to be located at the midpoints of their respective segments between the hip-knee and knee-ankle joints. 
Third, the dynamics of both legs are considered identical and independent of one another. Finally, nonlinear dissipative forces-including mechanical friction within motor gearing, joint bearings, and cable drag-are neglected. 
Under this simplified framework, the system dynamics become more manageable, allowing the analysis to focus on the primary driving forces: gravity and motor torques. 
The dynamics of the simplified bipedal robot can be analyzed using the single-leg model.
As a result, control laws can be computed efficiently and stability can be rigorously analyzed. 
Although these omissions may introduce minor deviations, such as drift, in experimental settings, the resulting model retains the essential dynamic characteristics of the real bipedal robot necessary for efficient control development.

The simplified model of our suspended bipedal robot is introduced below. Let \(P_1\) and \(P_2\) denote the hip and knee joints, respectively. Let \(l_1\) and \(l_2\) represent the lengths of the thigh and calf, respectively. Let \(P_{c1}\) and \(P_{c2}\) denote the centers of mass of the thigh and calf, respectively, and let \(m_{c1}\) and \(m_{c2}\) denote their corresponding masses. Finally, let \(m_1\) and \(m_2\) represent the masses of the servomotors at the hip and knee, respectively. Let \(\theta_1, \theta_2\) denote the angles of the thigh and calf with respect to the vertical (perpendicular to the ground) and \(\tau_1, \tau_2\) denote the torques generated by the servomotors at the hip and knee, respectively. For clarity, these notations are listed in Table~\ref{tab:table1}.
The dynamics of our single-leg model can be derived from  the Euler-Lagrange equation:
\begin{equation}
\frac{d}{dt} \left( \frac{\partial L}{\partial \dot{\theta}} \right) - \frac{\partial L}{\partial \theta} = \tau(t)
\end{equation}
where \(L=T-V\) is the Lagrangian of the kinetic energy \(T\) and the potential energy \(V\) of the system. 
\begin{table}[h!]
    \caption{Description of the Parameters in the Dynamic Model of the Left Lower Limb}
    \label{tab:table1}
    \centering
    \begin{tabular}{c|l}
        \hline
        \textbf{Symbol} & \textbf{Description} \\
        \hline
        $P_1$ & Hip joint\\
        $P_2$ & Knee joint \\
        $P_{c1}$ & Center of mass of thigh\\
        $P_{c2}$ & Center of mass of lower leg\\
        $l_1$ & Thigh length\\
        $l_2$ & Lower leg length\\
        $m_1$ & Mass of hip servo motor \\
        $m_2$ & Mass of knee servo motor\\
        $m_{c1}$ & Mass of thigh\\
        $m_{c2}$ & Mass of lower leg\\
        $\theta_1$ & Hip joint angle\\
        $\theta_2$ & Knee joint angle\\
        \hline
    \end{tabular}
\end{table}
\begin{figure}[htbp]
  \centering
  \subfloat[\label{fig:model_a}]{%
    \includegraphics[width=0.22\textwidth]{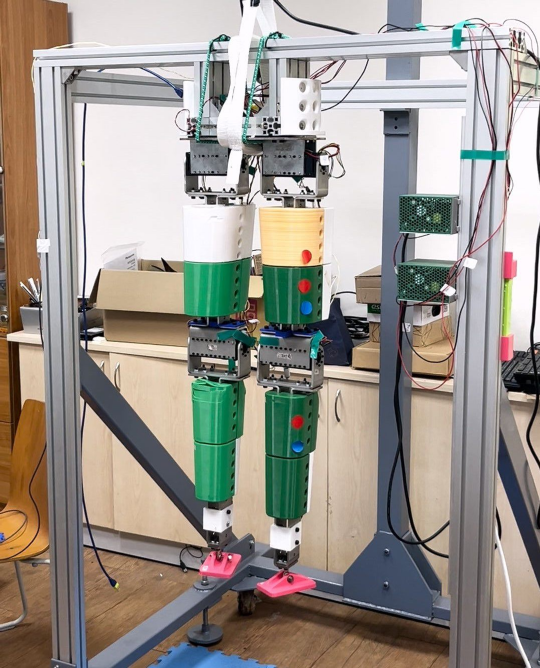}%
  }\hfill
  \subfloat[\label{fig:model_b}]{%
    \includegraphics[width=0.19\textwidth]{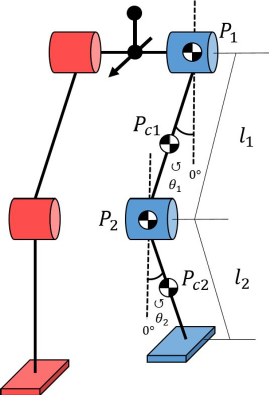}%
  }
  \caption{(a) Overall look of the bipedal robot. (b) Dynamic model of the bipedal robot, where the symbols are defined in Table~\ref{tab:table1}.}
  \label{fig:model}
\end{figure}
% The kinetic energy \(T\) is expressed as:
% \begin{align}
% T &= \frac{1}{8} m_{c1} l_1^2 \dot{\theta}_1^2 + \frac{1}{2} I_{c1} \dot{\theta}_1^2 + \frac{1}{2} m_2 l_1^2 \dot{\theta}_1^2\notag \\& +\frac{1}{2} m_{c2} \left[ l_1^2 \dot{\theta}_1^2 + \frac{1}{4} l_2^2 \dot{\theta}_2^2 + l_1 l_2 \dot{\theta}_1 \dot{\theta}_2 \cos(\theta_1 - \theta_2) \right] + \frac{1}{2} I_{c2} \dot{\theta}_2^2
% \end{align}
% where $I_{c1}=\frac{1}{12}m_{c1}l_1^2\text{ and }I_{c2}=\frac{1}{12}m_{c2}l_2^2$ are the inertia moments.
% and the potential energy
% \begin{equation}
%     \begin{aligned}
%         V = -m_2 g l_1 \cos \theta_1 - \frac{1}{2} m_{c1} g l_1 \cos \theta_1 \\
%         - m_{c2} g l_1 \cos \theta_1 - \frac{1}{2} m_{c2} g l_2 \cos \theta_2
%     \end{aligned}
% \end{equation}
The equations can be expressed in following matrix form:
\begin{equation}\label{eq:dynamics_state}
    \bm{\tau = M(\theta)\ddot{\theta}+V(\theta,\dot{\theta})\dot{\theta}+G_{sd}(\theta)\theta}
\end{equation}
where the state $\bm{\theta}=[\theta_1,\theta_2]^{\top}$ and the matrices \(\bm{M(\theta)}\), \(\bm{G_{sd}(\theta)}\), and \(\bm{V(\theta, \dot{\theta})}\) are defined as follows:
\begin{equation}\label{eq:matrix_M}
\begin{aligned}
&\bm{M(\theta)} =\\
&\begin{bmatrix} \frac{1}{4}m_{c1}l_1^2 + I_{c1} + m_2l_1^2+m_{c2}l_1^2 & \frac{1}{2}m_{c2}l_1l_2\cos(\theta_1 - \theta_2) \\ \frac{1}{2}m_{c2}l_1l_2\cos(\theta_1 - \theta_2) & \frac{1}{4}m_{c2}l_2^2 + I_{c2} \end{bmatrix}
\end{aligned}
\end{equation}
\begin{equation}\label{eq:matrix_V}
\begin{aligned}
&\bm{V(\theta)} =\\ &\begin{bmatrix} 0 & \frac{1}{2}m_{c2}l_1l_2\sin(\theta_1 - \theta_2)\dot{\theta}_2 \\ -\frac{1}{2}m_{c2}l_1l_2\sin(\theta_1 - \theta_2)\dot{\theta}_1 & 0 \end{bmatrix}   
\end{aligned}
\end{equation}
\begin{equation}\label{eq:matrix_Gsd}
\begin{aligned}
&\bm{G_{sd}(\theta)} =\\ 
&\begin{bmatrix} (m_2g l_1 + \frac{1}{2}m_{c1}g l_1+m_{c2}g l_1) \frac{\sin\theta_1}{\theta_1} & 0 \\ 0 & \frac{1}{2}m_{c2}g l_2 \frac{\sin\theta_2}{\theta_2} \end{bmatrix}.
\end{aligned}
\end{equation}
where $g$ is the gravitational constant, and the fraction term \(\frac{\sin\theta}{\theta}\) in the matrix \(\bm{G_{sd}(\theta)}\) is set to 1 as \(\theta\) tends to 0 to avoid numerical singularity.

To facilitate control and analysis of the bipedal robot system, consider a desired motion profile comprising the desired angles, angular velocities, and accelerations of motor joints. This profile must also satisfy the equation \eqref{eq:dynamics_state} and the desired torque $\bm{\tau}^d$ of the motor joints associated with the desired motion profile can be computed as follows:
\begin{equation}\label{eq:dynamics_desired}
    \bm{\tau}^d = \bm{M}(\bm{\theta}^d)\ddot{\bm{\theta}}^d + \bm{V}(\bm{\theta}^d)\dot{\bm{\theta}}^d + \bm{G}_{sd}(\bm{\theta}^d)\bm{\theta}^d.
\end{equation}
%where $\bm{\tau}^d$ is the desired torque of the motor joints associated with the desired motion profile. 
By subtracting \eqref{eq:dynamics_desired} from \eqref{eq:dynamics_state}, the  dynamics of the error state $\bm{e=\theta -\theta_d}$ can be written as follows:
%\begin{equation}\label{eq:dynamics_err1}
%\begin{aligned}
%    \bm{\tau} - \bm{\tau}^d &= \bm{M}(\bm{\theta})\ddot{\bm{\theta}} - \bm{M}(\bm{\theta}^d)\ddot{\bm{\theta}}^d + \bm{V}(\bm{\theta}, \dot{\bm{\theta}})\dot{\bm{\theta}} - \bm{V}(\bm{\theta}^d, \dot{\bm{\theta}}^d)\dot{\bm{\theta}}^d\\
%    &+ \bm{G}_{sd}(\bm{\theta})\bm{\theta} - \bm{G}_{sd}(\bm{\theta}^d)\bm{\theta}^d
%\end{aligned}
%\end{equation}
%The Eq. \eqref{eq:dynamics_err1} can be rewritten as following
%\begin{equation}\label{eq:dynamics_err2}
%\begin{aligned}
%   \bm{\tau} - \bm{\tau}^d &= \bm{M}(\bm{\theta})(\ddot{\bm{\theta}} - \ddot{\bm{\theta}}^d) + \bm{V}(\bm{\theta}, \dot{\bm{\theta}})(\dot{\bm{\theta}} - \dot{\bm{\theta}}^d) \\
%   +& \bm{G}_{sd}(\bm{\theta})(\bm{\theta} - \bm{\theta}^d) + \bm{g}_f(\bm{\theta}, \bm{\theta}^d, \dot{\bm{\theta}}, \dot{\bm{\theta}}^d, \ddot{\bm{\theta}}, \ddot{\bm{\theta}}^d)    
%\end{aligned}
%\end{equation}
\begin{equation}\label{eq:dynamics_err1}
\begin{aligned}
   \bm{\tau} - \bm{\tau}^d &= \bm{M}(\bm{e+\theta_d})\ddot{\bm{e}} + \bm{V}(\bm{e+\theta_d})\dot{\bm{e}} \\
   &+ \bm{G}_{sd}(\bm{e+\theta_d})\bm{e} + \bm{g}_f(\bm{\theta}, \bm{\theta}^d),    
\end{aligned}
\end{equation}
where
\begin{equation}
\begin{aligned}
    \bm{g}_f &= (\bm{M}(\bm{\theta}) - \bm{M}(\bm{\theta}^d))\ddot{\bm{\theta}}^d 
    + (\bm{V}(\bm{\theta}) - \bm{V}(\bm{\theta}^d))\dot{\bm{\theta}}^d \\
    &+ (\bm{G}_{sd}(\bm{\theta}) - \bm{G}_{sd}(\bm{\theta}^d))\bm{\theta}^d.
\end{aligned}
\end{equation}
Let $\bm{x}=[\bm{e},\bm{\dot{e}},\zeta]^{\top}$ denote the augmented state error where $\zeta$ is an augmented variable to handle the sourcing term $\bm{g_f}$. The dynamics of the error system Eq. \eqref{eq:dynamics_err1} can now be written in the following standard state-dependent coefficient (SDC) form for SDRE analysis:
\begin{equation}\label{eq:dynamics}
\bm{\dot{x} = A(x)x + B(x)u}
\end{equation}
with observer
\begin{equation}\label{eq:dynamics_observer}
\bm{y=Cx}
\end{equation}
where $\bm{u=\tau-\tau^d}$ is the control vector, the SDC matrix  $\bm{A(x)}$ and the control gain matrix $\bm{B(x)}$ are  %\eqref{eq:dynamics_A},\eqref{eq:dynamics_B}:
\begin{equation}\label{eq:dynamics_A}
\bm{A}(\bm{x}) = \begin{bmatrix}
\bm{0} & \bm{I}_2 & \bm{0} \\
-\bm{M}^{-1}\bm{G}_{sd} & -\bm{M}^{-1}\bm{V} & -\bm{M}^{-1}(\frac{\bm{g}_f}{\zeta}) \\
\bm{0} & \bm{0} & -\eta
\end{bmatrix}\in \mathbb{R}^{5\times5}
\end{equation}
and
\begin{equation}\label{eq:dynamics_B}
\bm{B}(\bm{x}) = \begin{bmatrix} \bm{0} & \bm{M}^{-1} & \bm{0} \end{bmatrix}^\top \in \mathbb{R}^{5\times2},
\end{equation}
respectively, and the observation matrix is considered as an identity matrix $\bm{C}=\bm{I}_5 \in \mathbb{R}^{5\times5}$ for simplicity. 
The parameter $\eta$ in the matrix $A$ is a tunable value greater than 0 and is typically chosen to be much smaller than 1. Once the control vector $\mathbf{u}$ is obtained, the corresponding torque input $\boldsymbol{\tau}$ for the servomotors can be calculated from $\boldsymbol{\tau}^d + \mathbf{u}$.\\
\noindent %A dynamic model that characterizes the relationship torque between joint torques and joint states of each leg is required to support both the SDRE control design and the torque reference used in the parameterized optimization. 
As the legs of the bipedal robot are actuated synchronously via a unified command sequence, the torques derived above for each leg are concatenated and integrated into an optimization framework. The desired motor command sequence is extracted from the optimization process to minimize the error between the reference torques $\bm{\tau}$ and the system torques obtained from the motor command sequence. Finally, track errors between the system model and the experimental results may arise due to  unmodeled friction and dissipative forces. An offline  PID-based controller is implemented to adjust the motor-command sequence and mitigate control losses. This hierarchical structure constitutes the three-stage cascade control algorithm detailed in the following section. 

\section{Control Methodology}
To effectively control the bipedal platform, we propose a three-stage control strategy: 
(i) SDRE Control: Reference torques are generated by applying SDRE control to the single-leg dynamics to approximate the torque profiles associated with the desired human lower-limb gaits recorded by a Vicon motion-capture system 
(ii) Parameterized Optimization: Using these reference torques as a target cost, an actuator-constrained optimization-incorporating command variable limits and motor characteristics is performed to derive an optimal motor command sequence and 
(iii) Acceleration Compensation: An acceleration-based compensation term is computed via a PID-LQR controller to mitigate deviations between the desired gaits and the experimental output. This framework utilizes an offline command generation and refinement process to ensure the consistent execution of predefined motion trajectories. The proposed three-stage control procedure is illustrated in Fig.~\ref{fig:BlockDiagram}. 
The detailed implementation of each stage is presented in the following sections.

\begin{figure}[htbp]
    \centering
    \includegraphics[width=0.85\linewidth]{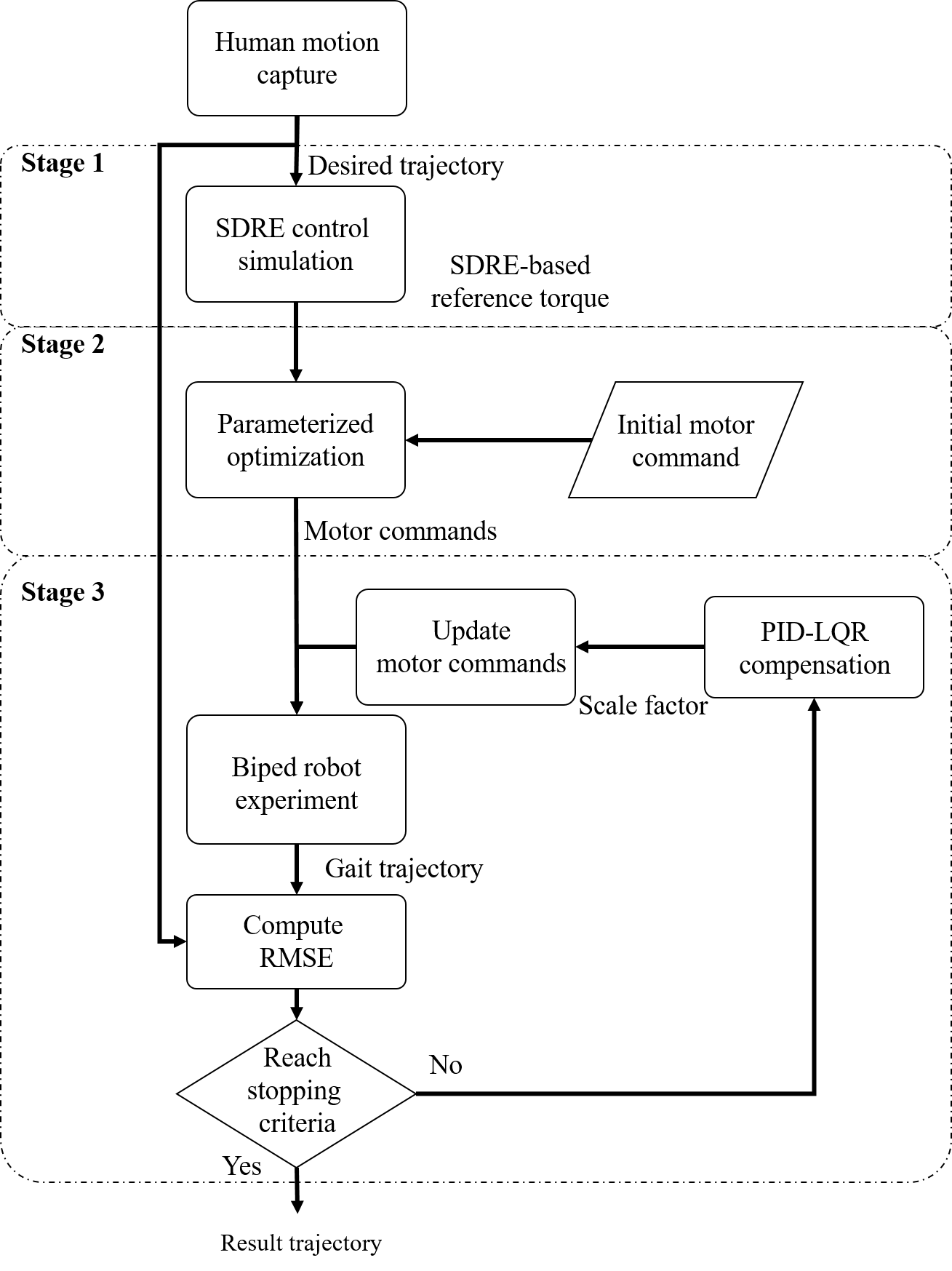}
     \caption{Flowchart of the proposed three-stage control procedure.}
    \label{fig:BlockDiagram}
\end{figure}

\subsection{SDRE Controller}
As mentioned in Section II, a model-based reference torque is essential to characterize the  behavior of the dynamics of the bipedal robot system.
%As established in Section II, the states of the bipedal robot's joints are driven toward the target motion profile via the torques derived from the dynamic system \eqref{eq:dynamics}. 
%These torques provide a model-based reference for the platform's motion reproduction. Since these torques are generated indirectly through commanded joint motions, a model-based reference is essential to characterize the ideal behavior of the system under its inherent dynamics.
To obtain this reference torque, we employ the SDRE framework where a prescribed tracking cost is minimized. In general, a quadratic cost function is considered in SDRE that balances tracking accuracy against control effort. A standard formulation of this cost function is given by \eqref{eq:costsdre}:
\begin{equation}\label{eq:costsdre}
    J = \int_0^{\infty} \left( \bm{x}(t)^{\top} \bm{Q} \bm{x}(t) + \bm{u}(t)^{\top} \bm{R} \bm{u}(t) \right) dt
\end{equation}
where \(\bm{Q}(\bm{x})\) and \(\bm{R}(\bm{x})\) are state-dependent weighting matrices. $\bm{Q}$ and $\bm{R}$ are associated with the state error and the control gain, respectively. In general, the matrix \(\bm{Q} \succeq 0\), \(\bm{R} \succ 0\) and both matrices are symmetric. For simplicity, these matrices are assumed to be constant, denoted \(\bm{Q}(\bm{x}) = \bm{Q}\) and \(\bm{R}(\bm{x}) = \bm{R}\) in this study. By minimizing the cost function in equation \eqref{eq:costsdre}, a suboptimal control law can be obtained as follows:
\begin{equation}\label{eq:u}
    \bm{u} = -\bm{R}^{-1} \bm{B}(\bm{x})^{\top} \bm{P}(\bm{x}) \bm{x}(t),
\end{equation}
where \(\bm{P}(\bm{x})\) is the solution of the state-dependent algebraic Riccati equation (ARE):
\begin{equation}\label{eq:riccati}
\begin{aligned} 
&\bm{P}(\bm{x})\bm{A}(\bm{x}) +\bm{A}(\bm{x})^{\top}\bm{P}(\bm{x}) \\
&\quad -\bm{P}(\bm{x})\bm{B}(\bm{x})\bm{R}^{-1} \bm{B}(\bm{x})^{\top}\bm{P}(\bm{x})+\bm{Q}=0. 
\end{aligned} 
\end{equation}
%and \(\bm{P}(\bm{x})\) is symmetric  positive definite matrix and the optimal control law of the equation \eqref{eq:dynamics} can be evaluated as:
%\begin{equation}\label{eq:u}
%    \bm{u} = -\bm{R}^{-1} \bm{B}%(\bm{x}(t))^{\top} \bm{P}(\bm{x}(t)) \bm{x}(t)
%\end{equation}
As a result, a model-based torque reference can be obtained at each point along the state trajectory, providing a robust benchmark for subsequent command generation.\\
To ensure the theoretical validity of the SDRE controller, one must show the solvability of the ARE and the stability of the error dynamics. Following the criteria established by Zhou~\cite{Zhou:1998}, it can be shown that ARE \eqref{eq:riccati} is solvable and dynamics \eqref{eq:dynamics} is stable for the single-leg model through the following theorems.
%are adopted as a verification condition for the existence of the solution and the stability of the dynamic system associated with the single-leg model.
\begin{theorem}~\cite{Zhou:1998}\label{thm:stable}
Suppose that $\bm{(A,B)}$ is stabilizable and $\bm{(A,C)}$ is detectable.
Then the Riccati equation in \eqref{eq:riccati} has a unique positive semi-definite solution $\bm{P}$.
Moreover, the matrix $\bm{(A - BR^{-1}B^{\top}P)}$ is stable.
\end{theorem}
%When the conditions stated in Theorem~\ref{thm:stable} hold, the ARE \eqref{eq:riccati} has a unique solution $P$. 
Now, by applying the control rule \eqref{eq:u} to the equation  \eqref{eq:dynamics}, clearly, the state error $\bm{x}$ satisfies $    \bm{\dot{x}} =(\bm{A} - \bm{B}\bm{R}^{-1}\bm{B}^{\top}\bm{P})\bm{x}$
%\begin{equation}\label{eq:error_dynamic}
%\begin{aligned}
%    \bm{\dot{x}} 
%    &= \bm{A}\bm{x} + \bm{B}(-\bm{R}^{-1}\bm{B}^{\top}\bm{P})\bm{%x}\\
  %  &= (\bm{A} - \bm{B}\bm{R}^{-1}\bm{B}^{\top}\bm{P})\bm{x}.
%\end{aligned}  
%\end{equation}
Consequently, Theorem~\ref{thm:stable} guarantees that the system is asymptotically stable. 
Next, we prove that the assumptions in Theorem 1 hold for the dynamics of the single-leg model by using the Hautus Lemma~\cite{Anderson:1990}.
%and assuming $\bm{M}(\theta)$ is always invertible.} %The remaining step is to verify that the state-dependent matrices used in this SDRE formulation satisfy the stabilizability and detectability requirements.
%Since $\bm{A}$ and $\bm{B}$ in \eqref{eq:dynamics_A} and \eqref{eq:dynamics_B} contain $\bm{M}^{-1}(\theta)$, the formulation is considered over the motion range where $\bm{M}(\theta)$ remains nonsingular.
%Under this regularity condition, the stabilizability of $(\bm{A},\bm{B})$ and detectability of $(\bm{A},\bm{C})$ are verified in Theorems~\ref{thm:stabilizable} and~\ref{thm:detectable} using the Hautus Lemma~\cite{Anderson:1990}.

\begin{theorem}\label{thm:stabilizable}
The pair $(\bm{A,B})$ is stabilizable for matrices $\bm{A}$ and $\bm{B}$ in \eqref{eq:dynamics_A} and \eqref{eq:dynamics_B}, respectively.
\end{theorem}
\begin{IEEEproof}
    \noindent To verify stabilizability, the equivalent Hautus condition in \cite{Anderson:1990} is applied:
     A pair $(\bm{A},\bm{B})$ is stabilizable if every left eigenvector $\bm{w} \neq 0$ of $A$
    associated with an eigenvalue $\lambda$ satisfying $\mathrm{Re}(\lambda)\ge 0$
    also satisfies $\bm{w^{\top}B} \neq \bm{0}$.\\
    Suppose, for contradiction, that there exists a vector
    \[
    \bm{w} = [w_1, w_2, w_3, w_4, w_5]^{\top} \neq \bm{0}.
    \]
    satisfying
    \[
    \bm{w}^{\top}\bm{A} = \lambda \bm{w}^{\top}, \qquad
    \mathrm{Re}(\lambda)\ge 0, \qquad
    \bm{w^{\top}B} = \bm{0}.
    \]
    The goal is to show that these assumptions necessarily lead to $\bm{w}=0$.
    According to the form of $\bm{B}$ in \eqref{eq:dynamics_B}, one has 
    \begin{equation}
        \bm{w}^{\top} \bm{B} = \bm{w}^{\top} \begin{bmatrix} \bm{0} & \bm{M^{-1}}& \bm{0} \end{bmatrix}^\top = [w_3, w_4] \bm{M^{-1}} = [0, 0].
    \end{equation}

    Now, consider
    \begin{equation}
    \bm{w}^{\top} \bm{A} = \bm{w}^{\top} \begin{bmatrix} \bm{0} & \bm{I} & \bm{0}\\ -\bm{M^{-1} G_{sd}} & -\bm{M^{-1} V} & \bm{-M^{-1}g_f}\\ \bm{0} & \bm{0} & -\bm{\eta}\end{bmatrix}.        
    \end{equation}
    Since the left block column and right block column are:
    \begin{equation}
    \begin{aligned}
    \bm{w}^{\top} \begin{bmatrix} \bm{0} & -\bm{M^{-\top} G_{sd}} & \bm{0}\end{bmatrix}^\top &= -[w_3, w_4] \bm{M^{-1}G_{sd}} = [0, 0],\\
    \bm{w}^{\top} \begin{bmatrix} I & -\bm{M}^{-\top} \bm{V} & \bm{0}\end{bmatrix}^\top &= [w_1, w_2] - [w_3, w_4] \bm{M}^{-1} \bm{V} \\&= [w_1, w_2], \mbox{ and}\\
    \bm{w}^{\top} \begin{bmatrix} \bm{0} & -\bm{M^{-\top} g_f} & -\bm{\eta}\end{bmatrix}^\top &= -\eta w_5,
    \end{aligned}
    \end{equation}
    we have
    \begin{equation}
    \begin{aligned}
    \bm{w}^{\top} \bm{A} &= [0, 0, w_1, w_2, -\eta w_5] \\
          &= [\lambda w_1, \lambda w_2, \lambda w_3, \lambda w_4, \lambda w_5],
    \end{aligned}
    \end{equation}
    by the assumption that \( \bm{w}^{\top} \bm{A} = \lambda \bm{w}^{\top} \).
    The first two components of $\bm{w^{\top}A}$ leads to 2 possible cases (1) $ \lambda = 0 $ and (2) $[w_1,w_2]=[0,0]$. Since case (1) lead to the result that $[w_1,w_2]=[0,0]$, we only need to consider case (2).
    For case (2), to show that \( \bm{w = 0} \), we have
    \begin{equation}
    [\lambda w_1, \lambda w_2, \lambda w_3, \lambda w_4, \lambda w_5] = [0,0,0,0,-\eta w_5].
    \end{equation}
    Note that $\eta>0$ is a parameter chosen by user, and $\lambda\neq0$ in case 2.  Therefore, the equality $\lambda w_5=-\eta w_5$ holds only if $w_5=0$ and it implies that $\bm{w=0}$.\\
\end{IEEEproof}

\begin{theorem}\label{thm:detectable}
    The pair $(\bm{A},\bm{C})$ is detectable for matrices $\bm{A}$ and $\bm{C}$ in \eqref{eq:dynamics_A} and \eqref{eq:dynamics_observer} , respectively.
\end{theorem}
\begin{IEEEproof}
To establish detectability, consider an equivalent condition from
\cite{Anderson:1990}: the pair $(\bm{A},\bm{C})$ is detectable if every nonzero left
eigenvector $\bm{w}$ of $\bm{A}$ associated with an eigenvalue $\lambda$ satisfying
$\mathrm{Re}(\lambda) \ge 0$ also satisfies $\bm{Cw} \neq \bm{0}$. Since the matrix $\bm{C} = \bm{I_{5}}$, obviously, 
\[
\bm{Cw} = \bm{I_{5} w} = \bm{w} \neq \bm{0},
\]
for any eigenvector $\bm{w} \neq \bm{0}$. Hence, the pair $(\bm{A},\bm{C})$ is detectable.
\end{IEEEproof}
\noindent The above Theorems   confirm that the SDRE control law is executable and the state error diminishes over time. Although the SDRE controller can drive the system state $\bm{x}$ to the desired state $\bm{x}^d$ theoretically, these control gains may not be feasible in practice due to inaccuracies in system modeling and constraints in actuators. To address this, in the following section, an optimal parameterization technique is introduced to approximate SDRE torques within the feasible command space of the motors, thereby converting the torque profiles into actuator-feasible command trajectories.

\subsection{Parameterized Optimization}
Although the SDRE controller provides theoretically optimal torque inputs, these idealized signals do not account for the practical characteristics and physical constraints of the joint actuators. 
To translate SDRE torque commands into executable motor inputs, a command structure compatible with servomotor actuation is required. 
Trapezoidal motion profiles are widely adopted in control systems due to their structural simplicity, continuous velocity characteristics, and low parameter count. 
Their analytical form supports efficient constraint evaluation and facilitates real-time execution. 
Building on these advantages, prior studies have adapted trapezoidal profiles for specialized applications~\cite{Chettibi:2006, Cusimano:2022}.
In this study, we represent the joint angular velocity using a parameterized piecewise-linear velocity model.

Unlike conventional trapezoidal profiles with fixed three-phase structures, our piecewise-linear parameterization offers greater flexibility by allowing both segment shapes and switching instants to vary. This adaptability is necessary because optimal torques generally cannot be perfectly replicated by rigid trapezoidal profiles, and the duration of command windows must be optimized to satisfy the physical constraints of the servomotors. Under this parameterization, the trajectory is determined by three sets of optimization parameters: command update instants, profile velocities, and profile accelerations. Additionally, the SDRE-generated angle profile is imposed as the target position at each command instant to ensure consistency with the desired gait. Based
on these optimized command instants, feasible motor-actuator commands can be
executed to approximate the reference torques derived from SDRE (the stage 1) 
while respecting the physical limits of the servomotors.

The proposed approach is motivated by the hardware specifications of the PH54-200-S500-R actuator~\cite{Robotis}. This motor executes a symmetric trapezoidal velocity profile comprising acceleration, constant-velocity, and deceleration phases. The resulting trajectory is governed by the target position, profile velocity, and profile acceleration. Crucially, the actuator supports an override mechanism that allows a new command to modify an in-progress profile before completion while maintaining velocity continuity. This capability provides the mechanical foundation for the piecewise-linear model developed here. To ensure high-fidelity motion reproduction, the command update instants are included as optimization variables. This addresses limitations found in prior studies, such as Lan et al.~\cite{Lan:2021a}, where imposing complete trapezoidal profiles over fixed intervals made it difficult to accurately track complex human motion. By optimizing the timing of these commands, we increase the system's adaptability to diverse motion profiles.

\noindent During optimization, a full trapezoidal pattern may not be finished in the time segment between two consecutive command update instants. A situation that the next command update instant falls within the constant-velocity phase, at the $k$-th segment, is illustrated in Fig.~\ref{fig:diagram_wt}.
\begin{figure}[ht!]
    \centering
    \begin{tikzpicture}[x=1cm,y=0.68cm]
        % Axes
        \draw[->] (0,0) -- (4,0) node[right] {$t$};
        \draw[->] (0,0) -- (0,3.55) node[above] {$\omega$};

        % Dashed lines for velocity levels and command instants
        \draw[dash pattern=on 3pt off 3pt] (0,2) -- (4,2);
        \draw[dash pattern=on 3pt off 3pt] (0,3) -- (4,3);
        \draw[dash pattern=on 3pt off 3pt] (1,0) -- (1,3);
        \draw[dash pattern=on 3pt off 3pt] (3,0) -- (3,3);

        % Time labels
        \node[below] at (0,0) {$\xi_0^k$};
        \node[below] at (1,0) {$\xi_1^k$};
        \node[below] at (3,0) {$\xi_0^{k+1}$};

        % Velocity labels
        \node[left] at (0,2) {$\omega_0^k$};
        \node[left] at (0,3) {$\omega^k$};

        % Piecewise linear velocity profile
        \draw[thick] (0,2) -- (1,3) -- (3,3);
    \end{tikzpicture}
    \caption{Piecewise linear velocity parameterization model}
    \label{fig:diagram_wt}
\end{figure}

\noindent where $\xi_0^k$ and $\omega^k$  denote the command update instant and the profile velocity (i.e. the target constant velocity reached after the acceleration phase), respectively, and $\alpha^k$ denotes the profile acceleration that governs the rate of change of velocity from $\omega_{0}^k$ to $\omega^k$.
The time instant $\xi_1^k$ defined in Equation \eqref{eq:t1k} denotes the transition time at which the angular velocity reaches $\omega^k$ starting from $\omega_{0}^k$ based on the angular acceleration $\alpha^k$:
\begin{equation}\label{eq:t1k}
\xi_1^k :=
\begin{cases}
\xi_0^k + \dfrac{\omega^k - \omega_{0}^k}{\alpha^k}, & \text{if } \alpha^k \neq 0 \\[10pt]
\xi_0^k, & \text{if } \alpha^k = 0
\end{cases}
\end{equation}
Notice that $\xi_1^k$ is not a predefined variable but is implicitly determined by the optimization variables $\alpha^k$, $\omega^k$, and $\xi_0^k$, with the initial velocity $\omega_{0}^k$ which is obtained from the ending velocity of the previous segment to ensure the continuity of velocity.

\noindent With the piecewise linear velocity model established for each time segment, the motor angular velocity is parameterized over the prescribed duration of a desired motion gait, and the corresponding angular position is recovered by integration. Consider the time interval $[0,T]$ and a sequence of command update instants $\{\xi_0^i\}_{i=1}^n$, where $\xi_0^1=0$ and $\xi_0^n=T$. Within each subinterval $[\xi_0^k,\xi_0^{k+1}]$ for $k=1,2,\ldots,n-1$, the piecewise linear angular velocity is defined by
\begin{equation}\label{eq:wk}
\widetilde{\omega}(t)=
\begin{cases}
\omega_0^k+\alpha^k(t-\xi_0^k),
& \xi_0^k\leq t\leq\xi_1^k,\\
\omega^k,
& \xi_1^k<t\leq\xi_0^{k+1}.
\end{cases}
\end{equation}
The corresponding angle profile $\widetilde{\theta}(t)$ is obtained by integrating $\widetilde{\omega}(t)$ over each subinterval, using $\widetilde{\theta}(\xi_0^k)$ as the initial angle. The terminal angle of each phase is used as the initial angle of the next phase, which preserves position continuity throughout the motion.

The parameters $\omega^k$ and $\alpha^k$ are constrained by the physical speed and acceleration limits of the motor:
\begin{equation}\label{eq:parameter_constraint}
\begin{aligned}
\omega_{\min}
&\leq \omega^k \leq \omega_{\max},
&& k=1,\ldots,n,\\
\alpha_{\min}
&\leq \alpha^k \leq \alpha_{\max},
&& k=1,\ldots,n.
\end{aligned}
\end{equation}
Because the command update instants are determined by the optimization, the velocity construction of a segment may differ from the two phase form in \eqref{eq:wk}. For example, a segment may contain only an acceleration phase or may span a complete trapezoidal velocity profile. In these cases, $\widetilde{\omega}(t)$ is defined using the applicable phases, and $\widetilde{\theta}(t)$ is obtained through the same segmentwise integration procedure. The resulting motor angle and angular velocity profiles are substituted into the dynamic equation in \eqref{eq:dynamics_state} to compute the corresponding torque profile $\widetilde{\tau}(t)$ used to evaluate the optimization cost function.
Accordingly, the optimization problem for extracting the command sequence of the bipedal robot can be formulated. Notice that motor commands for hip and knee joints of the single-leg model are given at the same time instant for each leg. The parameters of one leg are grouped together for optimization. Let
\begin{equation}
    \begin{alignedat}{2}
  \bm{\xi}_{L} &= [\xi_{0,L}^1,\ldots,\xi_{0,L}^n]^{\top},
&\;\bm{\xi}_{R} &= [\xi_{0,R}^1,\ldots,\xi_{0,R}^n]^{\top} 
\end{alignedat}
\end{equation}
denote the vectors of the time instant parameter for the left leg and the right leg of the bipedal robot, respectively. 
The profile velocity and acceleration associated with $\bm{\xi}_{L}$ and $\bm{\xi}_{R}$ can be expressed as follows: 
\begin{equation}\label{eq:parameterL}
\bm{\omega}_{L} :=
    \bigl[\bm{\omega}_{HL}^{\top},\; \bm{\omega}_{KL}^{\top}\bigr]^{\top},
\\
\bm{\alpha}_{L} :=
    \bigl[\bm{\alpha}_{HL}^{\top},\; \bm{\alpha}_{KL}^{\top}\bigr]^{\top}
\end{equation}
\begin{equation}\label{eq:parameterR}
\bm{\omega}_{R} :=
    \bigl[\bm{\omega}_{HR}^{\top},\; \bm{\omega}_{KR}^{\top}\bigr]^{\top},
\\
\bm{\alpha}_{R} :=
    \bigl[\bm{\alpha}_{HR}^{\top},\; \bm{\alpha}_{KR}^{\top}\bigr]^{\top}
\end{equation}
with
\begin{equation}\label{eq:parameterLR}
\begin{alignedat}{2}
    \bm{\omega}_{HL} &= [\omega_{HL}^1,\ldots,\omega_{HL}^n]^{\top},
    &\;\bm{\omega}_{HR} &= [\omega_{HR}^1,\ldots,\omega_{HR}^n]^{\top}\\
    \bm{\alpha}_{HL} &= [\alpha_{HL}^1,\ldots,\alpha_{HL}^n]^{\top},
    &\;\bm{\alpha}_{HR} &= [\alpha_{HR}^1,\ldots,\alpha_{HR}^n]^{\top}\\
    \bm{\omega}_{KL} &= [\omega_{KL}^1,\ldots,\omega_{KL}^n]^{\top},
    &\;\bm{\omega}_{KR} &= [\omega_{KR}^1,\ldots,\omega_{KR}^n]^{\top}\\
    \bm{\alpha}_{KL} &= [\alpha_{KL}^1,\ldots,\alpha_{KL}^n]^{\top},
    &\;\bm{\alpha}_{KR} &= [\alpha_{KR}^1,\ldots,\alpha_{KR}^n]^{\top},
\end{alignedat}
\end{equation}
where the subscripts HL and KL  correspond to the hip (H) and knee (K) joints of the left leg, while HR and KR correspond to those of the right leg. The optimization problem can now be formulated as
\begin{equation}\label{eq:paramter_optimal_problem}
%\begin{align*}
(\bm{\omega_\kappa^{*}},\bm{\alpha_\kappa^{*}},\bm{\xi_\kappa^{*}})
=
\underset {\bm{\omega_\kappa},\,\bm{\alpha_\kappa},\,\bm{\xi_\kappa}} {\text{arg min}}
\bm{J_\kappa}(\bm{\omega_\kappa},\bm{\alpha_\kappa},\bm{\xi_\kappa}),\\
\end{equation}

where $\kappa \in \{L,R\}$, subject to
\begin{equation}\label{eq:parameter_limit}
\begin{aligned}
\omega_{\min} \le {} & \omega_{\ell\kappa}^{k} \le \omega_{\max}
    && \text{for } \ell\in\{\text{H},\text{K}\},\; k=1,\ldots,n,\\
\alpha_{\min} \le {} & \alpha_{\ell\kappa}^{k} \le \alpha_{\max}
    && \text{for } \ell\in\{\text{H},\text{K}\},\; k=1,\ldots,n,\\
\end{aligned}
\end{equation}
and $0< \xi_{0,\kappa}^i < \xi_{0,\kappa}^j <T$ if $0<i<j<n$. 

For our single-leg model, a quadratic cost function $J_\kappa(\bm{\omega_\kappa},\bm{\alpha_\kappa},\bm{\xi_\kappa})$ shown in equation \eqref{eq:cost_param}:
\begin{equation}\label{eq:cost_param}
\begin{aligned}
J_{\kappa}(&\bm{\omega_\kappa},\bm{\alpha_\kappa},\bm{\xi_\kappa})
=
\\
&\left(
\frac{1}{T}
\int_{0}^{\top}
\bm{e}_\kappa\!\left(t;\bm{\omega_\kappa},\bm{\alpha_\kappa},\bm{\xi_\kappa}\right)^{\top}
\bm{W}\,
\bm{e}_\kappa\!\left(t;\bm{\omega_\kappa},\bm{\alpha_\kappa},\bm{\xi_\kappa}\right)
\, dt
\right)^{1/2}
\end{aligned}
\end{equation}
is considered. The diagonal matrix $\bm{W}\in\mathbb{R}^{2\times2}$ in \eqref{eq:cost_param} denotes the weighting matrix and $\bm{e_{\kappa}}$ is the torque error for the leg $\kappa \in \{L,R\}$ defined as $\bm{e}_{\kappa}(t;\bm{\omega_{\kappa}},\bm{\alpha_{\kappa}},\bm{\xi_{\kappa}}) = \bm{\tau_{\kappa}}(t) - \Tilde{\bm{\tau}}_{\kappa}(t;\bm{\omega_{\kappa}},\bm{\alpha_{\kappa}},\bm{\xi_{\kappa}})$, where $\bm{\tau_{\kappa}}(t)$ denotes the SDRE-based torque reference and $\Tilde{\bm{\tau}}_{\kappa}(t;\bm{\omega_{\kappa}},\bm{\alpha_{\kappa}},\bm{\xi_{\kappa}})$ represents the estimated torque induced by the parameterized command sequence. The optimal motor commands are obtained by minimizing the cost function using the MATLAB function \texttt{fmincon} with initial parameters, profile velocity and profile acceleration, being uniformly sampled from the SDRE-based state profiles presented in Section~II. The initial command update instants are uniformly distributed with a 0.25\,s interval to maintain adequate temporal resolution without causing convergence failures due to excessively dense segmentation.

Finally, the optimization procedure \eqref{eq:paramter_optimal_problem} generates individual command sequences for the hip and knee joints within each limb that share the same update instants as follows:
$C_{\ell,\kappa}=(\bm{\theta}_{\ell,\kappa}^*, \bm{\omega}_{\ell,\kappa}^*, \bm{\alpha}_{\ell,\kappa}^*)$, where $\ell \in \{H,K\}$, $\kappa \in \{L,R\}$, and $\bm{\theta}_{\ell,\kappa}^* = \bm{\theta}^d_{\ell,\kappa}$ is the desired trajectory angle associated with the optimized command instant $\bm{\xi}_{\kappa}^{*}$. These independently optimized command sequences $C_{\ell,\kappa}$ of the left and right leg are then temporally concatenated to form the complete and unified command set. The unified command is executed through the motor control board of the suspended bipedal robot to feasibly approximate the theoretical torque profiles from Section II, thereby enabling the reproduction of the desired human motion.

\subsection{Acceleration Compensation by offline PID}
In the previous subsection, SDRE-based torque profiles and the parameterized velocity model were used to derive the motor commands required for execution. However, these derivations do not account for the actual system response of the physical bipedal platform such as joint-level friction, actuator saturation, or command transition latencies and structure limitations, etc. To quantify the performance deviation between the theoretical model and the real-world system, a preliminary tracking test was conducted. The results of this test reveal a significant acceleration deficit, which motivates the offline refinement procedure introduced in this subsection.
The refinement uses a combined PID-LQR formulation to reduce discrepancies in joint angle and angular velocity. It introduces an acceleration scaling factor $\gamma(t)$ to adjust the Stage 2 acceleration command while preserving the optimized angle and velocity commands. With $\varepsilon_{\mathrm{exp}}(t)=\theta^d(t)-\theta^m(t)$ denoting the experimental tracking error and $\alpha(t)$ denoting the acceleration command obtained from optimization, the PID relation is written as
\begin{equation}
\gamma(t)\alpha(t)
=
K_p \varepsilon_{\mathrm{exp}}(t)
+
K_i \int_0^t \varepsilon_{\mathrm{exp}}(s)\,ds
+
K_d \dot{\varepsilon}_{\mathrm{exp}}(t)
\label{eq:PID}
\end{equation}
where $\gamma(t)$ is determined by the subsequent LQR formulation.

In this formulation, $\alpha(t)$ serves as a predetermined input and remains fixed during the subsequent derivation. By differentiating \eqref{eq:PID} and noting that the parameterized model uses piecewise-constant acceleration ($\dot{\alpha}(t)=0$ within each segment), the error dynamics can be expressed in the following state-space form:
\begin{equation}
\dot{\bm{\varepsilon}}_{\mathrm{exp}}
=
\bm{A}_e\bm{\varepsilon}_{\mathrm{exp}}+\bm{B}_e\dot{\gamma},
\label{eq:PIDviaLQR}
\end{equation}
where
\begin{equation}
\bm{\varepsilon}_{\mathrm{exp}}=
\begin{bmatrix}
\varepsilon_{\mathrm{exp}}\\
\dot{\varepsilon}_{\mathrm{exp}}
\end{bmatrix},
\quad
\bm{A}_e=
\begin{bmatrix}
0 & 1\\
-\tfrac{K_i}{K_d} & -\tfrac{K_p}{K_d}
\end{bmatrix},
\quad
\bm{B}_e=
\begin{bmatrix}
0\\
\tfrac{\alpha}{K_d}
\end{bmatrix}
\label{eq:PID_LQR_state_matrices}
\end{equation}
Here, $\bm{\varepsilon}_{\mathrm{exp}}$ denotes the experimental tracking error state, while $\bm{A}_e$ and $\bm{B}_e$ denote the error-state matrix and input matrix, respectively. The correction term $\dot{\gamma}$ serves as the control input in \eqref{eq:PIDviaLQR} and is computed using the LQR method. The scaling factor $\gamma(t)$ is then obtained by integrating $\dot{\gamma}(t)$.
Both the joint angle error and angular velocity error are included in the output. Accordingly, the observation matrix $\bm{C}_e$ is selected as the two-dimensional identity matrix $\bm{I}_2$.
% The observation model is defined as \eqref{eq:PID_LQR_observation_model}:
% \begin{equation}
% \bm{y}_e=\bm{C}_e\bm{x}_e,
% \qquad
% \bm{C}_e=\bm{I}_2 .
% \label{eq:PID_LQR_observation_model}
% \end{equation}
It should be noted that \eqref{eq:PIDviaLQR} does not represent a simplification of the original nonlinear robot dynamics. Instead, it is used to describe the state error evolution associated with the PID control structure, where the control action is expressed as a scaled acceleration input. This formulation is introduced to establish a state-space relationship between the tracking error and the correction term, enabling the application of LQR design. Furthermore, the proposed method does not assume time-invariant system behavior. The acceleration scaling factor is computed offline at each sampling instant, resulting in a time-varying sequence of correction terms. Consequently, the time-varying characteristics of the system are implicitly captured through this process.\\
Furthermore, since the compensated acceleration input $u_{PID}$ is included in the motor command, its feasibility must be considered in the command generation process. This requires that the correction input $\dot{\gamma}$ in \eqref{eq:PIDviaLQR} be well defined in the LQR formulation. Following the same procedure in SDRE analysis as shown in section II, the stabilizability of $(\bm{A}_e,\bm{B}_e)$ and the detectability of $(\bm{A}_e,\bm{C}_e)$ are verified using the Hautus Lemma \cite{Anderson:1990}. Once these conditions are satisfied, Theorem~\ref{thm:stable} provides the corresponding solvability condition for the Riccati equation.

\begin{theorem}\label{thm:PID_LQR_stabilizable_detectable}
Consider the error-state representation in \eqref{eq:PIDviaLQR}, with
$\bm{A}_e$ and $\bm{B}_e$ defined in
\eqref{eq:PID_LQR_state_matrices}, and let
$\bm{C}_e=\bm{I}_2$ denote the observation matrix. If
$\alpha\neq 0$ and $K_d\neq 0$, then the pair
$(\bm{A}_e,\bm{B}_e)$ is stabilizable. Moreover, the pair
$(\bm{A}_e,\bm{C}_e)$ is detectable.
\end{theorem}

\begin{IEEEproof}
By the Hautus Lemma \cite{Anderson:1990}, the pair
$(\bm{A}_e,\bm{B}_e)$ is stabilizable if
\begin{equation}
\operatorname{rank}
\begin{bmatrix}
\lambda\bm{I}_2-\bm{A}_e & \bm{B}_e
\end{bmatrix}
=2,
\quad
\forall \lambda\in\mathbb{C},\ \operatorname{Re}(\lambda)\geq 0 .
\label{eq:PID_LQR_Hautus_stabilizable}
\end{equation}
Substituting $\bm{A}_e$ and $\bm{B}_e$ from
\eqref{eq:PID_LQR_state_matrices} into
\eqref{eq:PID_LQR_Hautus_stabilizable} gives
\begin{equation}
\begin{bmatrix}
\lambda\bm{I}_2-\bm{A}_e & \bm{B}_e
\end{bmatrix}
=
\begin{bmatrix}
\lambda & -1 & 0\\
\tfrac{K_i}{K_d} & \lambda+\tfrac{K_p}{K_d} & \tfrac{\alpha}{K_d}
\end{bmatrix}.
\label{eq:PID_LQR_Hautus_stabilizable_matrix}
\end{equation}
The second and third columns of
\eqref{eq:PID_LQR_Hautus_stabilizable_matrix} form the following minor:
\begin{equation}
\det
\begin{bmatrix}
-1 & 0\\
\lambda+\tfrac{K_p}{K_d} & \tfrac{\alpha}{K_d}
\end{bmatrix}
=
-\tfrac{\alpha}{K_d}.
\label{eq:PID_LQR_Hautus_stabilizable_minor}
\end{equation}
Since $\alpha\neq 0$ and $K_d\neq 0$,
\eqref{eq:PID_LQR_Hautus_stabilizable_minor} is nonzero for all
$\lambda$. Therefore, \eqref{eq:PID_LQR_Hautus_stabilizable} holds and
$(\bm{A}_e,\bm{B}_e)$ is stabilizable.

The detectability of $(\bm{A}_e,\bm{C}_e)$ is established by the
corresponding Hautus condition:
\begin{equation}
\operatorname{rank}
\begin{bmatrix}
\lambda\bm{I}_2-\bm{A}_e\\
\bm{C}_e
\end{bmatrix}
=2,
\quad
\forall \lambda\in\mathbb{C},\ \operatorname{Re}(\lambda)\geq 0 .
\label{eq:PID_LQR_Hautus_detectable}
\end{equation}
Since $\bm{C}_e=\bm{I}_2$, the matrix in
\eqref{eq:PID_LQR_Hautus_detectable} becomes
\begin{equation}
\begin{bmatrix}
\lambda\bm{I}_2-\bm{A}_e\\
\bm{C}_e
\end{bmatrix}
=
\begin{bmatrix}
\lambda & -1\\
\tfrac{K_i}{K_d} & \lambda+\tfrac{K_p}{K_d}\\
1 & 0\\
0 & 1
\end{bmatrix}.
\label{eq:PID_LQR_Hautus_detectable_matrix}
\end{equation}
The last two rows of
\eqref{eq:PID_LQR_Hautus_detectable_matrix} form $\bm{I}_2$.
Consequently, the matrix has full column rank:
\begin{equation}
\operatorname{rank}
\begin{bmatrix}
\lambda\bm{I}_2-\bm{A}_e\\
\bm{C}_e
\end{bmatrix}
=2,
\quad
\forall \lambda\in\mathbb{C}.
\label{eq:PID_LQR_Hautus_detectable_rank}
\end{equation}
Therefore, $(\bm{A}_e,\bm{C}_e)$ is detectable.
\end{IEEEproof}
% ==================================================================
Theorem~\ref{thm:PID_LQR_stabilizable_detectable} ensures that the LQR correction input $\dot{\gamma}$ in \eqref{eq:PIDviaLQR} can be obtained. %Let $\bar{\dot{\gamma}}$ denote the maximum absolute value of $\dot{\gamma}$ over the finite execution interval $[0,T]$. 
Let $\beta=max(|\dot{\gamma(t)}|)$ for $t\in[0,T]$ where $T$ is the total execution time. Obviously, the scaling factor $\gamma(t)$ is bounded during the execution time since  
\begin{equation}
    |\gamma(t)|
    \leq
    |\gamma(0)|+\beta T,
    \qquad
    0\leq t\leq T .
    \label{eq:gamma_upper_bound}
\end{equation}
Since $\alpha(t)$ is constrained in the parameterized optimization stage and $u_{PID}(t)=\gamma(t)\alpha(t)$, the compensated acceleration input $u_{PID}(t)$ remains bounded. As a result, the acceleration terms in the motor command sequence can be updated by the PID-LQR algorithm while satisfying the motor acceleration constraints. Furthermore, the above compensation process is repeated until the stopping criteria are satisfied, i.e., when the maximum absolute angular error becomes smaller than the prescribed tolerance or when the number of compensation iterations reaches the prescribed maximum number of iterations.
\section{Experimental Results}
This section details the experimental validation of the proposed control framework through human gait trajectory tracking. The bipedal robot platform utilizes Robotis PH54-200-S500-R actuators for joint motion, integrated into a lightweight framework of 3D-printed components. Motor control is facilitated via a custom C\# application developed in Visual Studio 2022, with commands transmitted through a USB-connected U2D2 interface. The theoretical components, including the SDRE control and parameterized optimization, are implemented in MATLAB R2020b. 

% 2026.08.05
% \textcolor{lightbrown}{\sout{To evaluate performance, two baseline methods are implemented for comparison: Model Predictive Control (MPC) ~\cite{Chen:2018} and Improved Particle Swarm Optimization-based PID (IPSO-PID) ~\cite{Liu:2021} The MPC-based method computes control inputs through a quadratic programming formulation, while the IPSO-PID approach adaptively tunes PID parameters using an improved particle swarm optimization strategy. These baselines provide a benchmark for assessing tracking accuracy and repeatability under identical experimental conditions.}}

\subsection{Human Gait Data and SLLM Mapping}
Desired trajectories are obtained from a Vicon motion-capture system. Fig.~\ref{fig:vicon_skeleton}(a) illustrates the lower-limb skeleton captured by the Vicon system. To ensure kinematic consistency between the captured human data and the bipedal model, the Vicon skeleton is first reduced to six discrete joints to form a Simplified Lower Limb Model (SLLM) following the method described by Hsu~\cite{Hsu:2025}, as shown in Fig.~\ref{fig:vicon_skeleton}(b).
\begin{figure}[htbp]
  \centering
  \subfloat[]{%
    \includegraphics[width=0.22\textwidth]{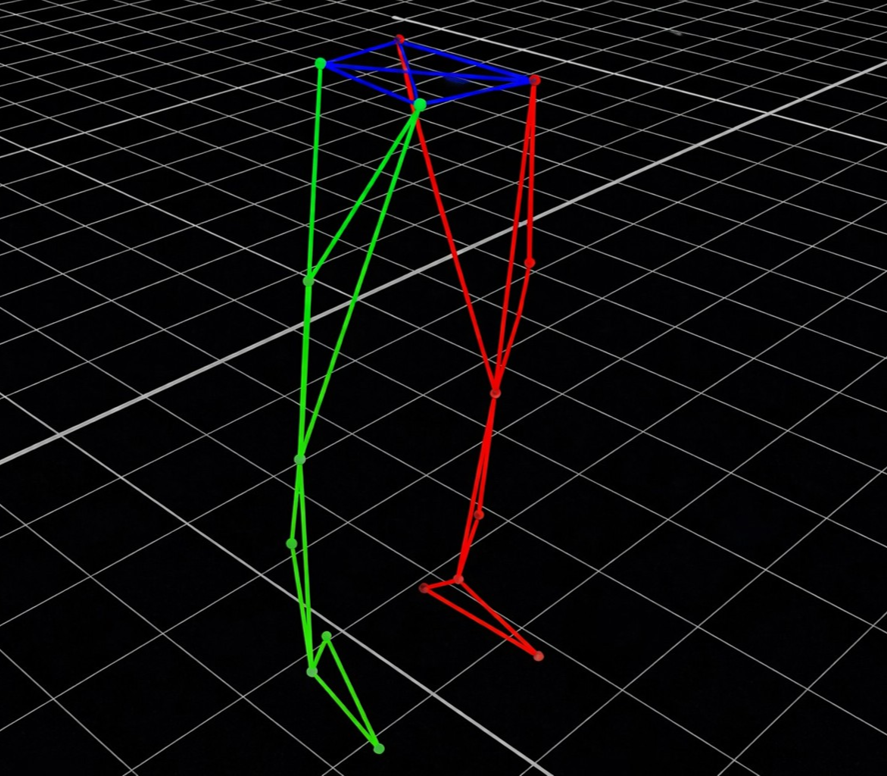}%
  }\hfill
  \subfloat[]{%
    \includegraphics[width=0.22\textwidth]{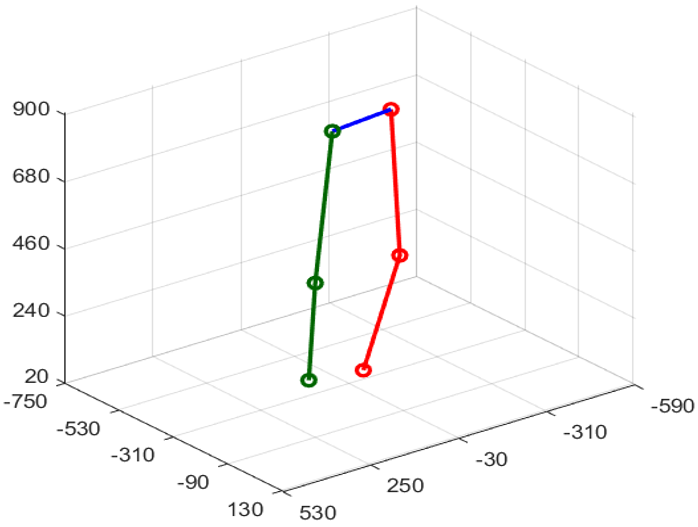}%
  }
  \caption{(a) Skeleton of lower limbs obtained from Vicon motion-capture data. (b) Simplified skeleton extracted from the Vicon motion-capture data.}
  \label{fig:vicon_skeleton}
\end{figure}
Each leg in the SLLM is projected onto a reference motion plane to filter out-of-plane noise and ensure planar consistency. The normal vector $\mathbf{n}$ of the reference motion plane is determined by averaging the instantaneous normal vectors over $T$ frames:
\begin{equation}
\mathbf{n} = \frac{1}{T}\sum_{i=1}^{\top} 
\frac{(\mathbf{h}_i-\mathbf{k}_i)\times(\mathbf{a}_i-\mathbf{k}_i)}
{\left\|(\mathbf{h}_i-\mathbf{k}_i)\times(\mathbf{a}_i-\mathbf{k}_i)\right\|}
\label{eq:normal_vector}
\end{equation}
where $\mathbf{h}_i$, $\mathbf{k}_i$, and $\mathbf{a}_i$ represent the positions of the hip, knee, and ankle markers in frame $i$, respectively. The projection operator $\mathbf{P}$ is then defined as:
\begin{equation}
\mathbf{P} = \mathbf{I} - \mathbf{n}\mathbf{n}^\top.
\label{eq:projection}
\end{equation}
Desired hip and knee angles are computed from the relative orientations of the adjacent limb segments after projection. These angles serve as the reference trajectories for the suspended bipedal robot.

\subsection{Setup and Performance Evaluation}
The SDRE controller is applied to generate optimal torque profiles for test motions, including walking and squatting. The physical parameters of the robot are set as $l_1 = 0.251$\,m, $l_2 = 0.28$\,m, $m_1 = 0.876$\,kg, and $m_2 = 0.876$\,kg, with  $m_{c1} = 2.89$\,kg and $m_{c2} = 3.242$\,kg. The weighting matrices in the cost function \eqref{eq:costsdre} are chosen to prioritize angular velocity tracking:
\[
\bm{Q} = \text{diag}(10,\,10,\,100,\,100,\,1), \quad \bm{R} = \text{diag}(20,\,20), 
\]
and the weighting matrix in equation \eqref{eq:cost_param} for parameterized optimization is defined as $\bm{W} = \bm{I}_{2\times2}$. To ensure that the generated motion commands remain physically feasible, the parameter limits in \eqref{eq:parameter_limit} are configured according to the PH54-200-S500-R actuator specifications \cite{Robotis}. Specifically, the velocity limits $\omega_{\max}$ and $\omega_{\min}$ are set to $\pm50^\circ/\mathrm{s}$, and the acceleration limits $\alpha_{\max}$ and $\alpha_{\min}$ are set to $\pm1000^\circ/\mathrm{s}^2$. 
Furthermore, joint motion ranges are constrained to align with the mechanical design of the platform. The hip joint is limited to $[-50^\circ, 50^\circ]$, while the knee joint is restricted to $[-20^\circ, 75^\circ]$. These constraints prevent mechanical interference and ensure safe operation during motion reproduction. Following optimization, the command sequences $C_{\ell,\kappa}$ ($\ell \in \{H,K\}, \kappa \in \{L,R\}$) are executed, and the feedback angles for each joint are recorded. Next, for the acceleration compensation procedure, the PID parameters are set to $K_p = 10^{-2}$, $K_i = 10^{-3}$, and $K_d = 50$ and the LQR weighting matrices for computing the acceleration correction are defined as $\bm{Q} = \text{diag}(1,1,1,1)$ and $\bm{R} = \text{diag}(1,10)$. Finally, after the acceleration scaling factor $\gamma(t)$ is calculated by the offline PID-LQR controller, the modified command sequence $C_{\ell,\kappa}^m$ is obtained and executed on the bipedal platform. 
Experimental snapshots of the squatting and walking trials are presented in Fig.~\ref{fig:expSnapShot}. The corresponding joint angular trajectories are illustrated in Fig.~\ref{fig:expAnglePlot}, where the red lines denote the reference SLLM trajectories and the blue lines represent the reproduced joint angles. Additionally, the shaded blue regions depict the bounding envelope between the maximum and minimum values across repeated trials, highlighting the high repeatability of the proposed control framework.
\begin{figure}[!t]
    \centering
    \subfloat[]{%
        \includegraphics[width=0.9\linewidth]{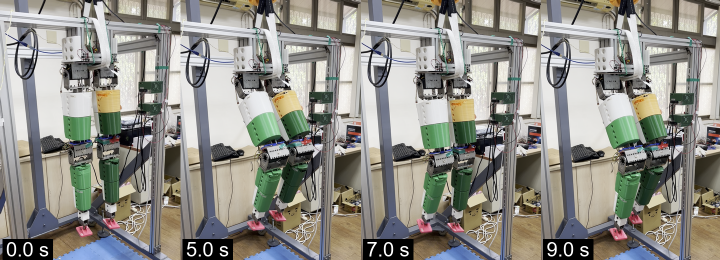}%
    }\hfill
    \subfloat[]{%
        \includegraphics[width=0.9\linewidth]{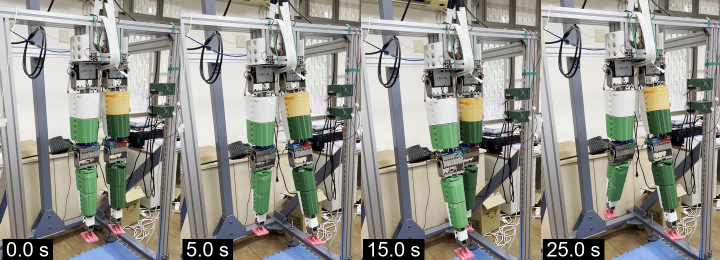}%
    }
    \caption{Snapshots of the suspended bipedal robot during experiments: (a) squatting; (b) walking.
}
    \label{fig:expSnapShot}
\end{figure}

\begin{figure*}[!t]
    \centering
    \subfloat[]{%
        \includegraphics[width=0.45\linewidth]{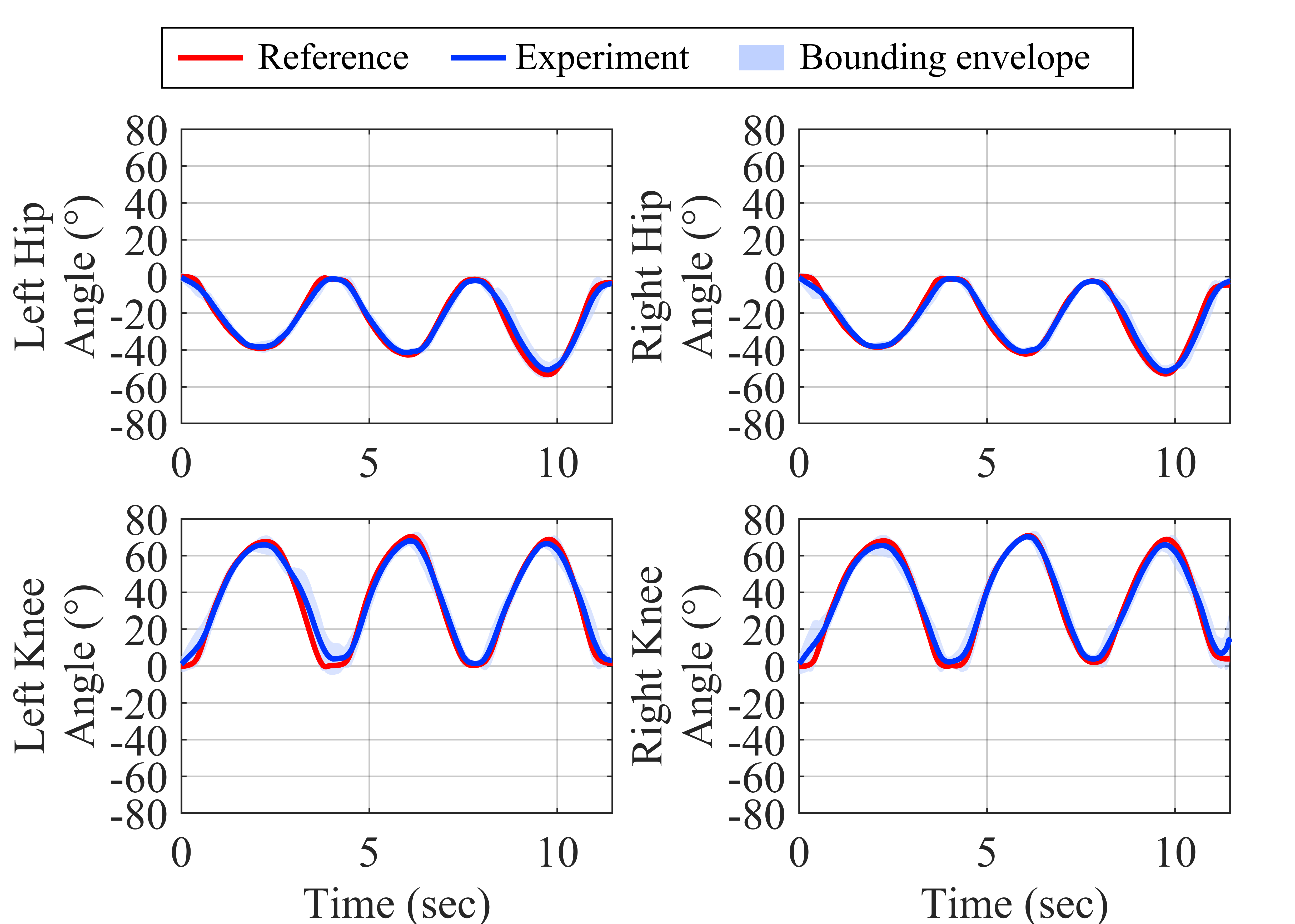}%
    }\
    \subfloat[]{%
        \includegraphics[width=0.45\linewidth]{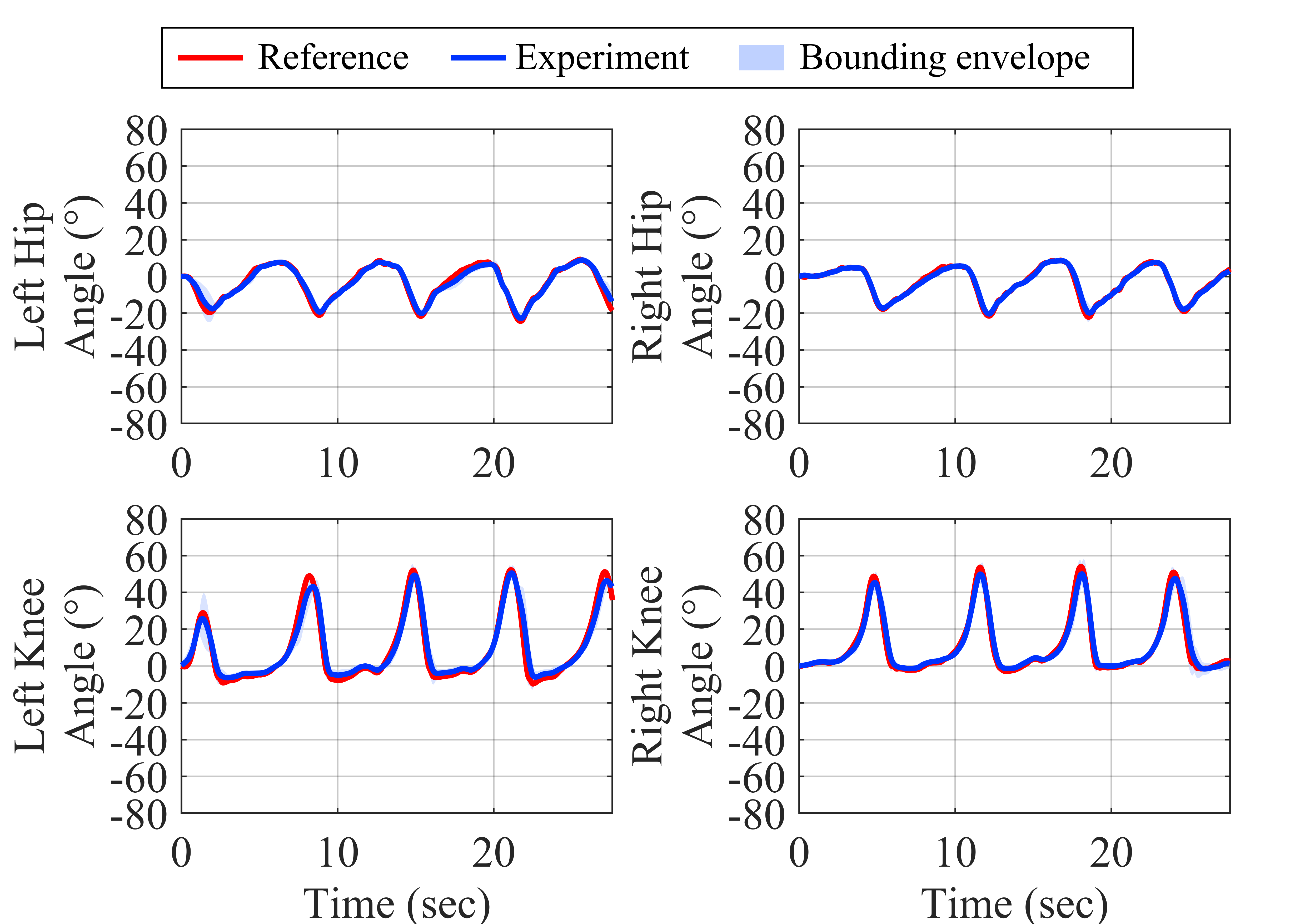}%
    }
    \caption{Comparison of reference and experimental joint angular trajectories for the suspended bipedal robot: (a) squatting; (b) walking.}
    \label{fig:expAnglePlot}
\end{figure*}

\begin{figure*}[t]
    \centering
    \subfloat[]{%
        \includegraphics[width=0.45\linewidth]{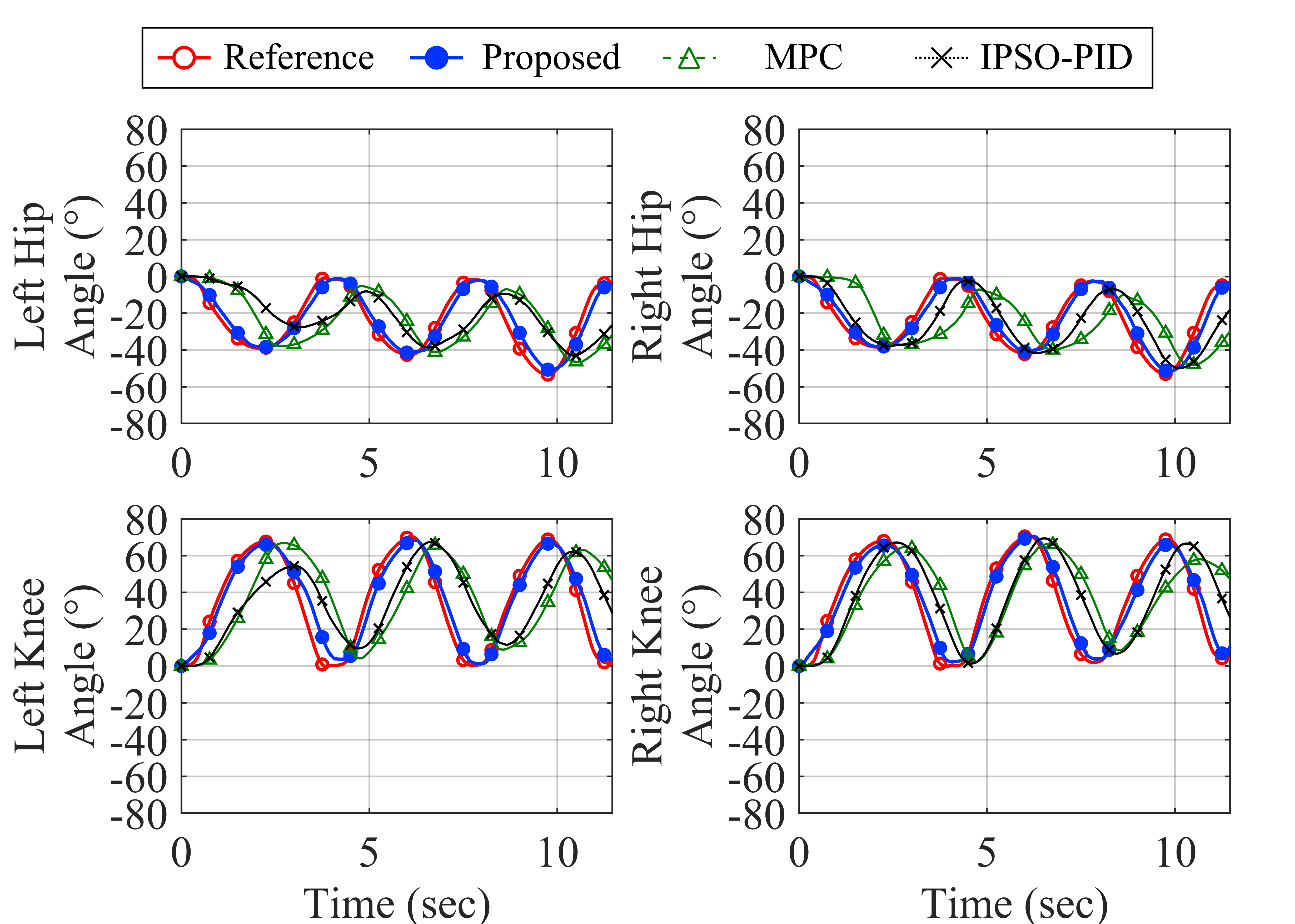}%
    }
    \subfloat[]{%
        \includegraphics[width=0.45\linewidth]{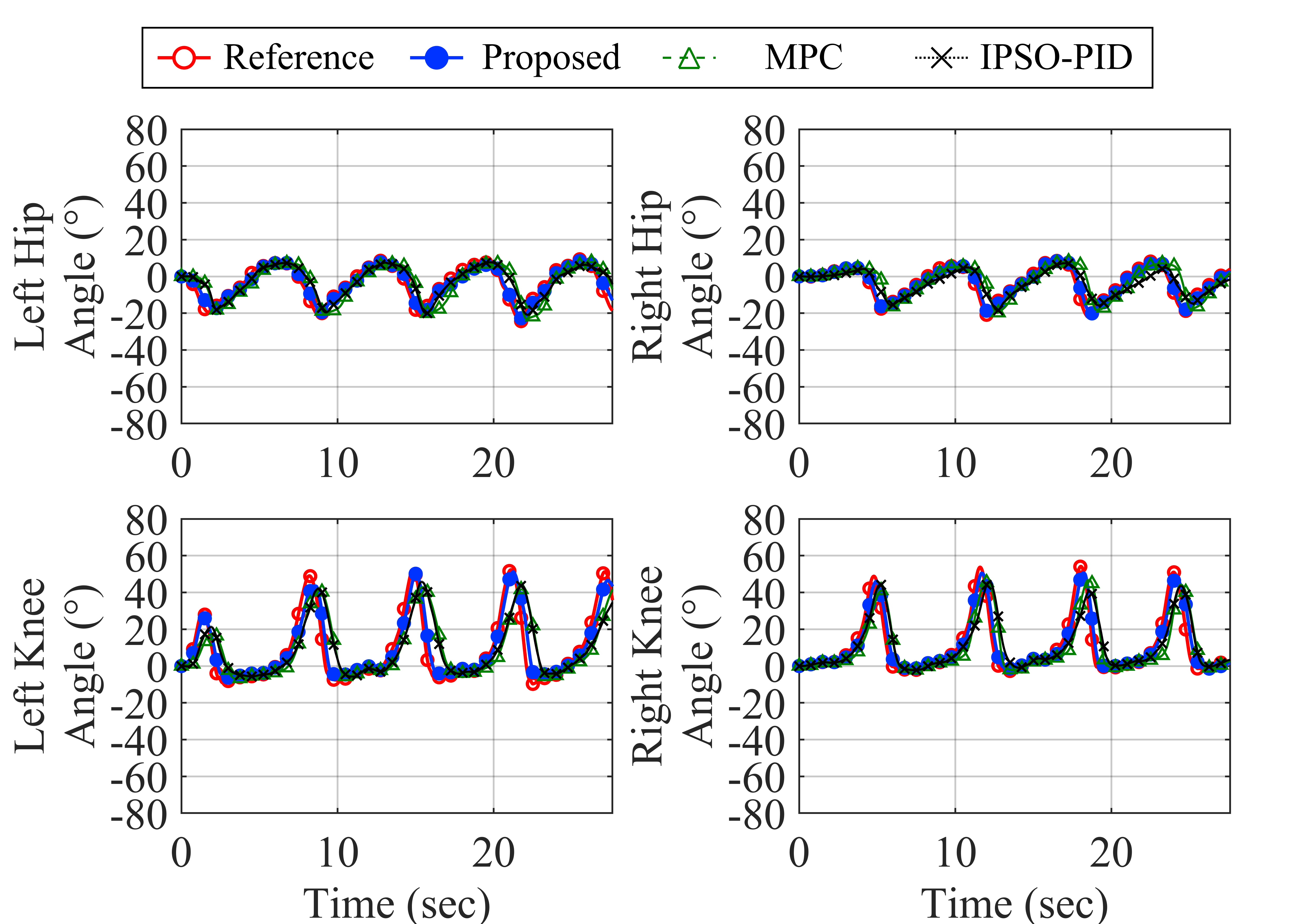}%
    }
    \caption{Experimental comparison of angular trajectories under the proposed method, MPC, and IPSO-PID for (a) squatting and (b) walking.}
    \label{fig:expAnglePlot2}
\end{figure*}

\begin{figure*}[t]
    \centering
    \subfloat[]{%
        \includegraphics[width=0.45\linewidth]{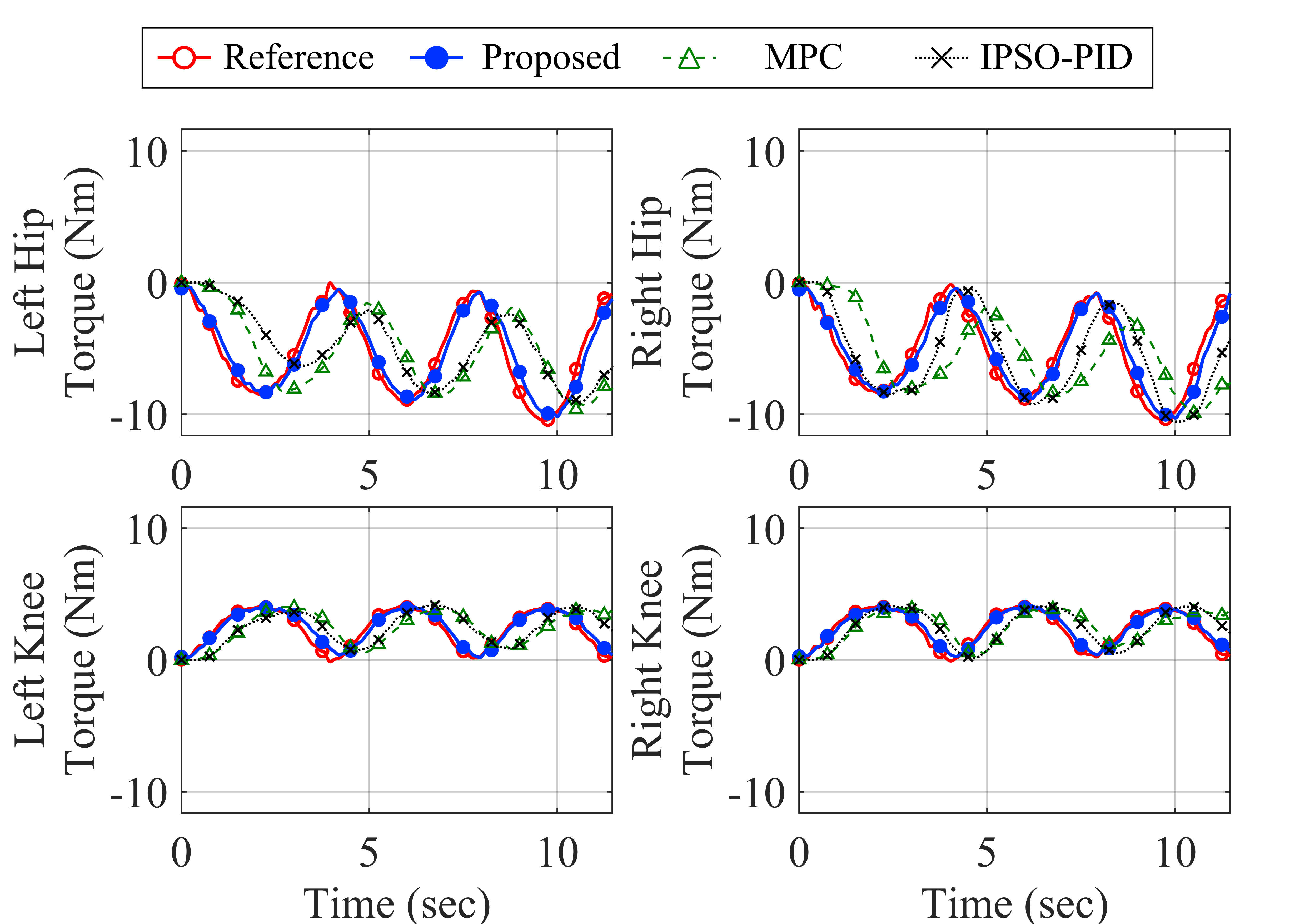}%
    }
    \subfloat[]{%
        \includegraphics[width=0.45\linewidth]{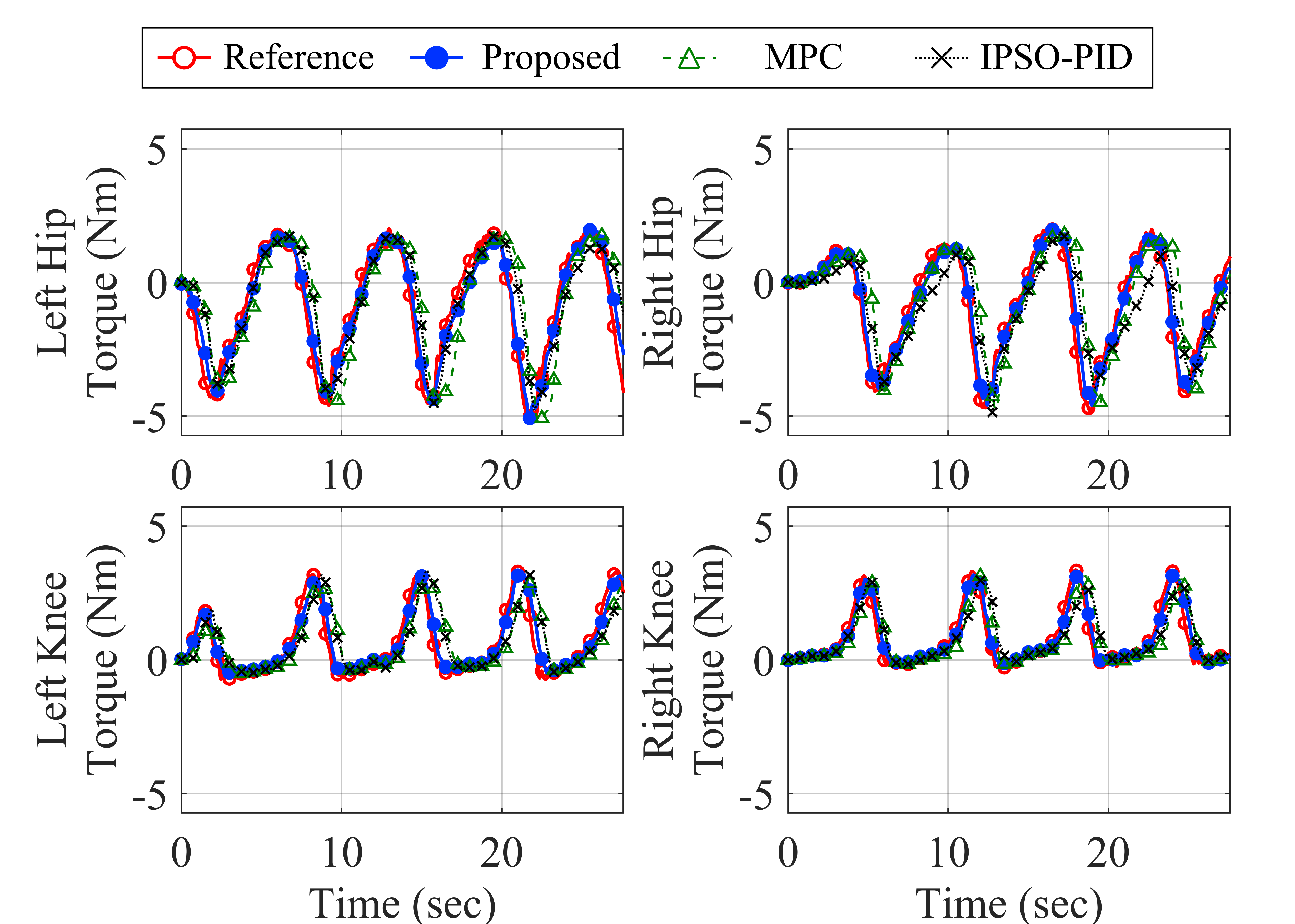}%
    }
    \caption{Experimental comparison of model-based joint torque trajectories under the reference motion, proposed method, MPC, and IPSO-PID for (a) squatting and (b) walking.}
    \label{fig:expTorquePlot}
\end{figure*}
\noindent The model-based torque trajectories in Fig.~\ref{fig:expTorquePlot} are computed from the dynamic relation in \eqref{eq:dynamics_state} based on the joint trajectories, angular velocities, and angular accelerations associated with each method. The reference curve represents the torque profile obtained by substituting the reference SLLM trajectory into the robot dynamic model. The curves of the proposed method, MPC, and IPSO-PID represent the torque profiles induced by their reproduced joint motions under the same dynamic model. Therefore, Fig.~\ref{fig:expTorquePlot} is used to examine whether each command generation method preserves the torque pattern required by the captured reference motion.
\subsection{Quantitative Analysis and Repeatability}
To evaluate the tracking accuracy and execution consistency of the proposed framework, a comparative experimental study was conducted using MPC~\cite{Chen:2018} and IPSO-PID~\cite{Liu:2021} as baseline controllers. All three methods executed the same walking and squatting reference trajectories over $M=10$ repeated trials under the same hardware and experimental conditions.
\noindent To quantify performance, three statistical metrics are evaluated for each joint: the average root mean square error ($\overline{\mathrm{RMSE}}$), the largest trial RMSE ($\mathrm{RMSE}_{\max}$), and the standard deviation of the trial RMSE values ($\mathrm{STD}$). For the $k$-th trial, the joint angle RMSE is computed as
\begin{equation}
(\mathrm{RMSE})_k
=
\sqrt{
\frac{1}{N}
\sum_{i=1}^{N}
\left(
\theta_i^k-\theta_i^{\mathrm{d}}
\right)^2
},
\end{equation}
where $\theta_i^k$ is the joint angle executed in time sample $t_i$, $\theta_i^{\mathrm{d}}$ is the desired reference angle, and $N$ is the number of samples in each trial. The average RMSE, the largest trial RMSE and the standard deviation of the trial RMSE values in the $M$ trials are defined as
\begin{equation}
\overline{\mathrm{RMSE}}
=
\frac{1}{M}
\sum_{k=1}^{M}
(\mathrm{RMSE})_k,
\end{equation}
\begin{equation}
\mathrm{RMSE}_{\max}
=
\max_{k=1,\ldots,M}
(\mathrm{RMSE})_k,
\end{equation}
\begin{equation}
\mathrm{STD}
=
\sqrt{
\frac{1}{M}
\sum_{k=1}^{M}
\left(
(\mathrm{RMSE})_k-\overline{\mathrm{RMSE}}
\right)^2
}.
\end{equation}

The results in Table~\ref{tab:table2} report $\overline{\mathrm{RMSE}}$ values of $1.7^\circ$--$5.5^\circ$ for walking and $4.2^\circ$--$6.4^\circ$ for squatting. Across all repeated trials, Table~\ref{tab:table3} shows that the largest RMSE of the trial and the STD of the RMSE values of the trial range from $1.7593^\circ$ to $6.8950^\circ$ and from $0.0367^\circ$ to $0.3220^\circ$, respectively.
% ========================================================================
\begin{table}[t]
\centering
\caption{Average joint angle RMSE between the desired and experimentally reproduced trajectories.}
\label{tab:table2}
\small
\begin{tabular}{c|c|c|c}
\hline
\textbf{Motion} &
\textbf{Side} &
\textbf{Part} &
\(\overline{\mathrm{\mathbf{RMSE}}}\) \textbf{($^\circ$)} \\
\hline
\multirow{4}{*}{Squatting}
& Left  & Hip  & 4.2322 \\ 
& Left  & Knee & 6.4162 \\ 
& Right & Hip  & 4.2520 \\ 
& Right & Knee & 5.4709 \\ 
\hline
\multirow{4}{*}{Walking}
& Left  & Hip  & 2.3029 \\ 
& Left  & Knee & 5.5297 \\ 
& Right & Hip  & 1.6835 \\ 
& Right & Knee & 4.5135 \\
\hline
\end{tabular}
\end{table}
% ========================================================================
% For the $k$-th trial, the RMSE is computed as:
% \begin{equation}
% (\mathrm{RMSE})_k = \sqrt{\frac{1}{N} \sum_{i=1}^{N} \left( \theta_{i}^k - \theta^{\mathrm{d}}_i \right)^2}
% \end{equation}
% where $\theta_{i}^k$ is the reproduced angle at time sample $t_i$, $\theta^{d}_i$ is the target angle, and $N$ is the total number of samples. The averaged metric is defined as:
% \begin{equation}
% \overline{\mathrm{RMSE}} = \frac{1}{M} \sum_{k=1}^{M} (\mathrm{RMSE})_k.
% \end{equation}
% % ==========================================================================
% \textcolor{lightbrown}{The maximum RMSE used in the subsequent comparison is defined by \eqref{eq:maximum_trial_rmse} as the largest RMSE of the complete trajectory among the $M$ repeated trials:}
% \begin{equation}
% \mathrm{RMSE}_{\max}
% =
% \max_{k=1,\ldots,M}(\mathrm{RMSE})_k.
% \label{eq:maximum_trial_rmse}
% \end{equation}
% % ===========================================================================
% Additionally, we assess variability by calculating the standard deviation (STD) across trials:
% \begin{equation}
% \mathrm{STD} = \sqrt{\frac{1}{M} \sum_{k=1}^{M} \left( (\mathrm{RMSE})_k - \overline{\mathrm{RMSE}} \right)^2}.
% \end{equation}
% ==========================================================================
\begin{table}[t]
\centering
\caption{Maximum RMSE and Standard Deviation Over 10 Trials with the Same Command.}
\label{tab:table3}
\resizebox{0.45\textwidth}{!}{
\begin{tabular}{c|c|c|c|c}
\hline
\textbf{Motion} & \textbf{Side} & \textbf{Part} & \textbf{Max RMSE}($^\circ$) & \textbf{STD}($^\circ$) \\
\hline
\multirow{4}{*}{Squatting}
& Left  & Hip  & 4.9573 & 0.3220 \\ 
& Left  & Knee & 6.8950 & 0.2552 \\ 
& Right & Hip  & 4.4141 & 0.1371 \\ 
& Right & Knee & 5.6846 & 0.1765 \\ 
\hline
\multirow{4}{*}{Walking}
& Left  & Hip  & 2.6144 & 0.1113 \\ 
& Left  & Knee & 5.7767 & 0.1342 \\ 
& Right & Hip  & 1.7593 & 0.0367 \\ 
& Right & Knee & 4.8855 & 0.1844 \\
\hline
\end{tabular}%
}
\end{table}
%=============================================================================================
\begin{table*}[t]
\centering
\caption{Comparison of Joint Angle Maximum RMSE and Standard Deviation for the Control Methods}
\label{tab:table4}
\resizebox{0.75\textwidth}{!}{
\begin{tabular}{c|c|c|ccc|ccc}
\hline
\multirow{2}{*}{\textbf{Motion}} &
\multirow{2}{*}{\textbf{Side}} &
\multirow{2}{*}{\textbf{Part}} &
\multicolumn{3}{c|}{\textbf{Max RMSE ($^\circ$)}} &
\multicolumn{3}{c}{\textbf{STD ($^\circ$)}} \\
\cline{4-9}
 & & &
\textbf{Proposed} & \textbf{MPC} & \textbf{IPSO-PID} &
\textbf{Proposed} & \textbf{MPC} & \textbf{IPSO-PID} \\
\hline
\multirow{4}{*}{Squatting}
& Left  & Hip  &
\textbf{4.9573} & 23.1072 & 21.7738 &
\textbf{0.3220} & 0.9573 & 1.4650 \\
& Left  & Knee &
\textbf{6.8950} & 33.3585 & 30.9332 &
\textbf{0.2552} & 1.3899 & 1.9498 \\
& Right & Hip  &
\textbf{4.4141} & 22.3842 & 14.4919 &
\textbf{0.1371} & 0.9863 & 0.8421 \\
& Right & Knee &
\textbf{5.6846} & 29.2677 & 25.1407 &
\textbf{0.1765} & 1.5468 & 1.8900 \\
\hline
\multirow{4}{*}{Walking}
& Left  & Hip  &
\textbf{2.6144} & 7.7421 & 8.9949 &
\textbf{0.1113} & 0.2535 & 1.0652 \\
& Left  & Knee &
\textbf{5.7767} & 18.1610 & 16.6445 &
\textbf{0.1342} & 1.0665 & 1.3544 \\
& Right & Hip  &
\textbf{1.7593} & 7.8747 & 7.0862 &
\textbf{0.0367} & 0.3994 & 0.6051 \\
& Right & Knee &
\textbf{4.8855} & 12.9693 & 15.2319 &
\textbf{0.1844} & 0.7676 & 1.5836 \\
\hline
\end{tabular}
}
\end{table*}
%=============================================================================================
\begin{table*}[t]
\centering
\caption{Comparison of Joint Torque Maximum RMSE and Standard Deviation for the Control Methods}
\label{tab:table5}
\resizebox{0.75\textwidth}{!}{
\begin{tabular}{c|c|c|ccc|ccc}
\hline
\multirow{2}{*}{\textbf{Motion}} &
\multirow{2}{*}{\textbf{Side}} &
\multirow{2}{*}{\textbf{Part}} &
\multicolumn{3}{c|}{\textbf{Max RMSE (Nm)}} &
\multicolumn{3}{c}{\textbf{STD (Nm)}} \\
\cline{4-9}
 & & &
\textbf{Proposed} & \textbf{MPC} & \textbf{IPSO-PID} &
\textbf{Proposed} & \textbf{MPC} & \textbf{IPSO-PID} \\
\hline
\multirow{4}{*}{Squatting}
& Left  & Hip  &
\textbf{0.9607} & 4.4610 & 4.3505 &
\textbf{0.0589} & 0.1851 & 0.2856 \\
& Left  & Knee &
\textbf{0.4242} & 1.8206 & 1.8207 &
\textbf{0.0187} & 0.0729 & 0.1235 \\
& Right & Hip  &
\textbf{0.8725} & 4.3836 & 3.0073 &
\textbf{0.0267} & 0.1947 & 0.1782 \\
& Right & Knee &
\textbf{0.3464} & 1.6140 & 1.4690 &
\textbf{0.0149} & 0.0859 & 0.1083 \\
\hline
\multirow{4}{*}{Walking}
& Left  & Hip  &
\textbf{0.5728} & 1.6791 & 2.0110 &
\textbf{0.0212} & 0.0598 & 0.2650 \\
& Left  & Knee &
\textbf{0.3740} & 1.2091 & 1.1596 &
\textbf{0.0081} & 0.0763 & 0.0954 \\
& Right & Hip  &
\textbf{0.4053} & 1.7357 & 1.6058 &
\textbf{0.0109} & 0.0989 & 0.1364 \\
& Right & Knee &
\textbf{0.3096} & 0.8008 & 0.9874 &
\textbf{0.0139} & 0.0478 & 0.1029 \\
\hline
\end{tabular}
}
\end{table*}
%===========================================================================================
As illustrated by the comparative angular trajectory in Fig.~\ref{fig:expAnglePlot2} and torque responses in Fig.~\ref{fig:expTorquePlot}, a direct performance comparison with MPC and IPSO-PID is provided in Table~\ref{tab:table4} for joint angles and Table~\ref{tab:table5} for model-based joint torques.
%\textcolor{red}{As illustrated by the bounding envelopes in Fig.~\ref{fig:expAnglePlot2} and the comparative torque responses in Fig.~\ref{fig:expTorquePlot}, (no envelopes in figure 10 and 11)} 
The results show that the proposed method achieves lower tracking errors and trial-to-trial variations across all reported joints and motion tasks:
\begin{itemize}
    \item \textbf{Kinematic Tracking and Repeatability:}
    The proposed method achieves angular $\mathrm{STD}$ values ranging from $0.0367^\circ$ to $0.3220^\circ$, compared with $0.2535^\circ$ to $1.5468^\circ$ for MPC and $0.6051^\circ$ to $1.9498^\circ$ for IPSO-PID. In terms of maximum RMSE for repeated trials, the proposed framework reduces $\mathrm{RMSE}_{\max}$ by at least $62.3\%$ relative to MPC and $65.3\%$ relative to IPSO-PID for all joints and motion tasks reported. The corresponding angular $\mathrm{STD}$ is reduced by at least $56.1\%$ relative to MPC and $78.0\%$ relative to IPSO-PID.
    
    \item \textbf{Dynamic Torque Preservation:}
    As shown in Table~\ref{tab:table5}, the proposed method reduces the torque $\mathrm{RMSE}_{\max}$ by at least $61.3\%$ relative to MPC and $67.7\%$ relative to IPSO-PID. The corresponding torque $\mathrm{STD}$ is reduced by at least $64.5\%$ and $79.4\%$, respectively. These results indicate that the model-based torque profiles obtained from the proposed method more closely reproduce the reference torque patterns than those obtained from MPC and IPSO-PID.
\end{itemize}

\noindent %Based on the angle and torque comparisons, the proposed offline control strategy provides more consistent motion reproduction than the online baseline methods. 
Although online control methods are capable of adjusting control inputs in real time to respond to disturbances, their performance is inherently influenced by real-time feedback quality and solver behavior. In the context of the suspended bipedal robot platform, the primary objective is to reproduce predefined motion trajectories with high consistency over multiple repetitions, rather than to react to unpredictable disturbances. As a result, variations introduced by real-time feedback processing may lead to trial-to-trial discrepancies. In contrast, the proposed offline control strategy eliminates this dependency by fixing the control commands prior to execution. This makes our method more suitable for experimental scenarios where repeatability is critical, such as benchmarking and systematic evaluation of motion reproduction performance.
\section{Conclusion}
%\textcolor{lightbrown}{This study established a suspended bipedal robotic platform designed to serve as a controlled, repeatable environment for motion reproduction in exoskeleton-related research. 
%Reliable motion reproduction is achieved through a hybrid three-stage offline control strategy that balances theoretical optimality with practical execution constraints. 
%The framework integrates an SDRE-based torque profile as a theoretical reference, a parameterized velocity command model to ensure hardware feasibility, and an offline PID-LQR refinement stage to compensate for unmodeled dynamics and improve execution accuracy.
%Experimental results across walking and squatting tasks demonstrate that the proposed control method outperforms MPC and IPSO-PID in reproducing accurate and repeatable joint motion. 
%Using the proposed control algorithm, the reproduced joint motions for walking and squatting exhibit average angular errors of 1--3$^\circ$ and trial-to-trial standard deviations below 0.15$^\circ$. 
%The torque comparison further shows that the proposed method achieves lower maximum RMSE values than MPC and IPSO-PID in all reported cases.} \textcolor{red}{Suggestion: replace this paragraph by the blue part in the end of this section} \\
Experimental evaluations across walking and squatting tasks demonstrated that the proposed three-stage controller reproduces the reference trajectories more accurately and consistently than MPC and IPSO-PID.  The corresponding torque comparisons further indicate that the reproduced motions more closely preserve the torque patterns obtained from the dynamic model for the captured reference motions. These results indicate that the proposed method enables accurate and repeatable motion reproduction on the suspended bipedal robot platform.
Consequently, this work provides a low-risk, controlled experimental framework for the early-stage evaluation of knee-type exoskeletons without requiring direct human participation. At this stage, the system is positioned as a foundational motion reproduction platform. 
Although the current validation focuses on the standalone bipedal robot reproducing motion gaits captured by the Vicon tracking system and excludes the interaction force between the foot and the ground, the proposed control framework is designed for extensibility. Future work will focus on extending the system model to incorporate joint friction, ground interaction forces, and hardware coupling effects when physical exoskeleton units are mounted on the robot. Because the core three-stage control architecture, specifically the SDRE reference generation and PID-LQR compensation, is modular, it remains directly applicable to these expanded dynamic models through systematic parameter updates.
%%%%%%%%%%%%%%%%%%%%%%%%%%%%%%%%%%%%%%%%
\section*{Acknowledgments} 
\noindent This work was supported by the National Science and Technology Council (NSTC), Taiwan, under Grant No.~114-2115-M-A49-017-MY2. The corresponding author would also like to thank Yen-Rong Huang and Yun-Shan Chen for their assistance with the experiments.
%%%%%%%%%%%%%%%%%%%%%%%%%%%%%%%%%%%%%%%%

\end{document}